\documentclass{article}

\usepackage{arxiv}

\usepackage[utf8]{inputenc} 
\usepackage[T1]{fontenc}    
\usepackage{hyperref}       
\usepackage{url}            
\usepackage{booktabs}       
\usepackage{amsfonts}       
\usepackage{nicefrac}       
\usepackage{microtype}      
\usepackage{graphicx}
\usepackage{natbib}
\usepackage{doi}

\usepackage{researchpack}
\usepackage[dvipsnames]{xcolor}
\usepackage[capitalise,nameinlink]{cleveref}
\usepackage{multirow}
\usepackage{colortbl}
\usepackage{wrapfig}
\usepackage{pifont}

\usepackage{algorithm}
\usepackage[noend]{algpseudocode}

\hypersetup{
    colorlinks=true,
    linkcolor=CadetBlue,
    citecolor=CadetBlue,
}

\newcommand{\smallcitep}[1]{\scalebox{0.6}{\citep{#1}}}
\definecolor{tablegray}{RGB}{212, 212, 212}

\newcommand{\cmark}{\ding{51}}

\newcommand{\CCMARK}{\textcolor{SeaGreen}{\cmark\kern-.5em\cmark}}

\newcommand{\Act}{\ensuremath{s}}
\newcommand{\vAct}{\ensuremath{\vs}}
\newcommand{\calAct}{\ensuremath{\calS}}
\newcommand{\Deltaj}{\ensuremath{\delta{}\vw^{(j)}}}
\newcommand{\logdiffj}{\ensuremath{\delta\mu^{(j)}(\vAct)}}

\newcommand{\AXp}{\ensuremath{\mathsf{AXp}}\xspace}

\newcommand{\TopKAct}{\texttt{TopK-Act}\xspace}
\newcommand{\TopKWei}{\texttt{TopK-Wei}\xspace}
\newcommand{\TopKProd}{\texttt{TopK-Prod}\xspace}
\newcommand{\OriginalExactAlgorithm}{\texttt{EA}\xspace}
\newcommand{\OurExactAlgorithm}{\texttt{EA+}\xspace}

\newcommand{\PTwo}{\textbf{P2}\xspace}

\newcommand{\LaBo}{\texttt{LaBo}\xspace}
\newcommand{\LFCBM}{\texttt{LF-CBM}\xspace}
\newcommand{\VLGCBM}{\texttt{VLG-CBM}\xspace}

\newcommand{\GLasso}{\texttt{glasso}\xspace}
\newcommand{\ENet}{\texttt{enet}\xspace}

\newcommand{\CUB}{\texttt{CUB}\xspace}
\newcommand{\CelebA}{\texttt{CelebA}\xspace}
\newcommand{\Derm}{\texttt{Derm7pt}\xspace}
\newcommand{\WBC}{\texttt{WBCAtt}\xspace}

\title{When Are Concept Bottleneck Model \\ Explanations Faithful and Compact?}
\author{
    Stefano Teso \\
    CIMeC \& DISI \\
    University of Trento \\
    {\tt stefano.teso@unitn.it}
    \And
    Emanuele Marconato \\
    Department of Mathematical Sciences \\
    University of Copenhagen \\
    {\tt emanuele.marconato@math.ku.dk}
    \AND
    Steve Azzolin \\
    DISI \\
    University of Trento \\
    {\tt steve.azzolin@unitn.it}
    \And
    Antonio Vergari \\
    School of Informatics \\
    University of Edinburgh \\
    {\tt avergari@ed.ac.uk}
}

\renewcommand{\shorttitle}{When Are Concept Bottleneck Model Explanations Faithful and Compact?}

\hypersetup{
pdftitle={A template for the arxiv style},
pdfsubject={q-bio.NC, q-bio.QM},
pdfauthor={David S.~Hippocampus, Elias D.~Striatum},
pdfkeywords={First keyword, Second keyword, More},
}

\begin{document}
\maketitle

\begin{abstract}
	Concept bottleneck models (CBMs) are neural classifiers that allow to explain their decisions via  high-level concepts, potentially enabling understanding, steering and debugging.
    However, 
    their explanations are often derived \textit{heuristically}.
    Building on \textit{formal} explainability -- we argue they should also be faithful, \ie not misreport which concepts actually matter.
    We show that, for widespread CBM architectures, including recent VLM-based variants,
    faithful explanations must include all concepts in the bottleneck, compromising interpretability when this is large.  This result applies to \textit{both} heuristic and faithful-by-construction {formal explanations}.
    To encourage the existence of compact faithful explanations, we suggest \textit{i}) modeling concepts \textit{probabilistically} as binary or categorical random variables (rather than logits), and \textit{ii}) employing per-concept training-time sparsification via \textit{group lasso} (rather than regular elastic net).  We also
    extend algorithms from formal explainability to CBMs, and show they outperform natural heuristics in terms of guarantees and explanation size.
    Overall, our work warns against na\"ive interpretability claims and provides formal conditions and practical strategies for ensuring CBMs are as interpretable as advertised.
\end{abstract}

\section{Introduction}

Concept bottleneck models (CBMs) \citep{koh2020concept, ji2025comprehensive, knab2026s} are ``explainable-by-design'' neural classifiers that predict high-level concepts (\eg objects present in an input image) and use these to infer a label (\eg  the image's class), see \cref{fig:page-two} (left).
By enabling stakeholders to inspect the set of concept activations -- the titular \textit{bottleneck} -- CBMs aim to support interpretation \citep{bogani2026are}, steering \citep{shin2023closer}, and interactive debugging \citep{stammer2021right}.
Yet, reliable interaction requires CBMs to produce explanations that are both \textit{\textbf{compact}} \citep{ramaswamy2023overlooked, schrodi2025selective}, thus supporting understanding, and \textit{\textbf{faithful}} to the CBM's decision process, so as not to mislead stakeholders \citep{marques2023logic, teso2023leveraging}.

Neither property is guaranteed in practice.
On the one hand, bottlenecks can include hundreds or even thousands \citep{yang2023language} of activations, well beyond what stakeholders can grasp.
The go-to solution is to compute a small explanation -- that is, a subset of the bottleneck -- $\calE$ using heuristic algorithms, \eg by taking the $k$ highest activations, but see \cref{fig:page-two} and \cref{sec:compactness-faithfulness-algorithms} for alternatives \citep{yuksekgonul2023post, yang2023language, oikarinen2023label, srivastava2024vlgcbm}.

On the other, as we will show, \textit{heuristic explanations can be unfaithful}, \ie they may misreport which concepts do vs. do not determine the prediction.
Specifically, a subset $\calE$ that supports different predictions is insufficient to justify either \citep{marques2023logic, darwiche2023logic, soria2026formal}.  To see why, consider the heuristic explanation $\calE$ in \cref{fig:page-two} (right), computed on the \CUB bird classification dataset \citep{wah2011caltech}.  It supports predictions for two different classes, {\tt Vesper Sparrow} and {\tt Black-billed Cuckoo}.  Since the activations in $\calE$ are identical for both predictions, these cannot be determined by $\calE$ alone:  they must depend on activations \textit{outside} $\calE$.

We seek to understand when compact faithful explanations can be obtained.
As a first contribution, we demonstrate that this is not always possible:  unless a CBM is sufficiently constrained, \textit{\textbf{faithful explanations must include all concepts in the bottleneck}}, no matter what algorithm is used to compute them.
We show that this depends on how a CBM constrains the concept activations and the inference layer mapping concepts to labels and that many existing CBM architectures fail on both fronts.
As a second contribution, we discuss two simple but effective \textit{strategies that encourage CBMs to admit compact faithful explanations}, namely:
\textit{i}) modeling concept activations as probability distributions and not logits; and
\textit{ii}) applying structured, rather than unstructured, sparsity to the inference layer via group lasso.
As a third contribution, we propose \textit{replacing heuristics with algorithms from formal explainability} \citep{marques2023logic, darwiche2023logic}, which ensure faithfulness, and introduce \OurExactAlgorithm, an exact algorithm geared for compactness.
Our experiments in \cref{sec:experiments} highlight how our modeling strategies enable finding more compact faithful explanations for regular and VLM-based CBMs \citep{oikarinen2023label, yang2023language, debole2026if} and how \OurExactAlgorithm outperforms heuristic alternatives while remaining computationally affordable.

\textbf{Contributions.}  We:
\textit{i}) Study the faithfulness of CBM explanations, proving that, for many regular and VLM-based CBMs, faithful explanations cannot be compact (\cref{sec:faithfulness-vs-compactness});
\textit{ii}) From an algorithmic perspective, extend exact explainers from FXAI that ensure faithfulness to favor compactness (\cref{sec:compactness-faithfulness-algorithms});
\textit{iii}) From a modeling perspective, propose two simple but effective strategies that constrain CBMs to allow more compact faithful explanations (\cref{sec:solutions});
\textit{iv}) Present an extensive experimental evaluation showcasing the underlying issues and the benefits of our contributions (\cref{sec:experiments}).

\begin{figure}
    \centering
    \includegraphics[width=0.975\linewidth]{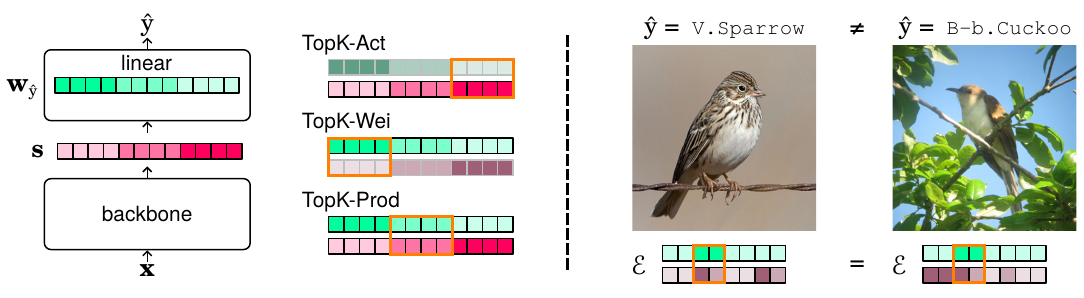}
    \caption{
        \textbf{Current heuristic explanation for CBMs can be unfaithful} as shown for a CBM combining a neural \textit{backbone} generating concept activations $\vAct \in \calAct$ and a linear \textit{inference layer} with per-class weights $\vw_y$, with $y \in [m]$. \textbf{Left:} Standard heuristics compute explanations $\calE$ by looking at the magnitude of activations, weights, or both, cf. \cref{eq:heuristics}.
        \textbf{Right:} A heuristic explanation $\calE$ that covers the $k = 2$ top weights: it justifies two decisions on the ground that the same two activations attain the same values.  Yet, the decisions differ:  the CBM must be looking at activations outside of $\calE$ to produce them.  Hence, $\calE$ is insufficient to fully justify either decision.
    }
    \label{fig:page-two}
\end{figure}

\section{Concept Bottleneck Models}
\label{sec:prelims}

\textbf{Notation.}  We denote scalar constants in lower-case ($x$), vector constants in bold ($\vx$), scalar and vector variables in upper-case ($X$, $\vX$), and sets in upper-case italics ($\calX$).
We abbreviate $\{1, \ldots, n\}$ to $[n]$ and, given a vector $\vx \in \bbR^n$, we write $\vx_\calI$ for its restriction to the indices in $\calI \subseteq [n]$.

\textbf{Concept bottleneck models.}  CBMs \citep{koh2020concept} are neural classifiers that combine a \textit{backbone} $g: \calX \to \calAct$ mapping inputs (\eg images) $\vx \in \calX$ to a bottleneck of $b$ concept \textit{activations} $\vAct \in \calAct$, and an \textit{inference layer} $f: \calAct \to \bbR^m$ mapping activations to $m$ per-class logits; see \cref{fig:page-two} (left).
The prediction $\hat y \in [m]$ is the highest-scoring class.
Concept activations $\vAct$ are traditionally modelled as \textit{independent logits} \citep{koh2020concept, shin2023closer, park2026analysis}, \ie $\calAct \equiv \bbR^b$, and only translated into probabilities for computing the loss.
Most commonly, the inference layer is a \textit{linear model} of the form $f(\vAct) = W \vAct + \vw_0$, where $W \in \bbR^{m \times b}$ and $\vw_0 \in \bbR^m$ are parameters.
CBMs define a loss for labels and one for concepts:
\[
    \label{eq:cbm-losses}
    \ell^Y(f, g) \; := \;
        \widehat\bbE_{(\vx_i, y_i)} \big[
            - \log p (y_i \mid \vx_i)
        \big],
    \quad
    \ell^C(f, g) \; := \;
        \widehat\bbE_{(\vx_i, \vc_i)} \big[
            - {\textstyle \frac{1}{b} \sum_{j \in [b]}} \log p(c_{ij} \mid \vx_i)
        \big].
\]
Here, $y_i \in [m]$ are task labels and $\vc_i \in \{0, 1\}^b$ concept labels; the distributions are $p(y \mid \vx) = \mathrm{softmax}((f \circ g)(\vx))_y$ and $p(c_j \mid \vx) = \mathrm{sigmoid}(g(\vx)_j)$; and the expectations are estimated on a densely annotated training set $\{ (\vx_i, \vc_i, y_i) \}$.  The concept loss aims to align concepts to human semantics, but see \cref{sec:related-work} for a discussion of failure cases.
The exact loss depends on whether the model is trained \textit{independently} (each module fit on corresponding labels), \textit{sequentially} (backbone fit on concept labels, inference layer on concept predictions), or \textit{jointly} (end-to-end with a convex combination of the two losses).

This backbone-plus-linear-classifier setup acts as a blueprint for most CBM variants -- see \citep{ji2025comprehensive, knab2026s} for overviews -- and for other concept-based architectures, like self-explainable \citep{alvarez2018towards} and prototype-based \citep{chen2019looks} neural nets.

\textbf{VLM-aided CBMs.}  Concept supervision is key for ensuring learned concepts are human-aligned \citep{bortolotti2025shortcuts, fokkema2026sample}, yet seldom available.
For datasets lacking it, ``VLM-CBMs'' bypass this issue by asking an LLM to generate a concept vocabulary (\ie textual descriptions of task-relevant concepts) and then obtaining weak concept supervision from a Vision-Language Model (VLM) like CLIP \citep{radford2021learning}, typically by measuring the similarity between concept descriptions and images in the VLM's embedding space \citep{yuksekgonul2023post, oikarinen2023label, yang2023language, tan2024explain, kazmierczak2026enhancing, debole2026if}.
These activations are used either to train a backbone or employed directly for inference.

\section{Compactness and Faithfulness}
\label{sec:compactness-faithfulness-algorithms}

By decomposing predictions into higher-level concepts, CBMs aim to facilitate mental simulation and reliability assessment \citep{rudin2019stop, narayanan2018humans, bogani2026are}, human intervention \citep{shin2023closer, espinosa2023learning}, and interactive debugging \citep{stammer2021right, bontempelli2022concept}.
Reliable interaction, however, requires CBMs to supply stakeholders with concept-level explanations that are both \textit{compact} -- so as to enable understanding -- and \textit{faithful} to their decision process -- so as not to mislead stakeholders.

\textbf{Compactness}.  Taking the \textit{whole} bottleneck to serve as explanation is rarely useful, as CBM vocabularies can comprise hundreds of entries \citep{lage2019human, barker2023selective, schrodi2025selective}, well beyond what stakeholders can grasp.
E.g., the default vocabulary of the \CUB bird classification task (\citet{wah2011caltech}, illustrated in \cref{fig:page-two}, right) comprises $312$ activations, and LLM-generated vocabularies for the same task count up to $10{,}000$ \citep{yang2023language}.

Practitioners often construct a smaller explanation by extracting a subset $\calE \subset [b]$ of concept activations.
Given a prediction $(\vAct, \hat{y})$ and a target size $k < b$, three well-known \textit{heuristic algorithms} for doing so (shown in \cref{fig:page-two}, middle) are picking the $k$ activations\footnote{Alternatively, one can fix a weight, activation, or score threshold $\tau \ge 0$. Our results extend to this case.} with the highest value \citep{chen2019looks, ley2026prototype}, the largest-magnitude weights \citep{yuksekgonul2023post, yang2023language}, or a combination thereof \citep{oikarinen2023label, srivastava2024vlgcbm}:
\[
    \label{eq:heuristics}
    \calE = \operatorname{argtop}_k \{ \Act_i \},
    \quad\quad
    \calE = \operatorname{argtop}_k \{ |w_{\hat y, i}| \},
    \quad\quad
    \calE = \operatorname{argtop}_k \{ |\Act_i \cdot w_{\hat y, i}| \}.
\]
One exception are Selective CBMs \citep{schrodi2025selective}, which construct $\calE$ via gating.
These heuristics are sometimes paired with \textit{training-time sparsification} penalties encouraging CBM predictions to rely only on few concepts.
We will detail these and Selective CBMs in \cref{sec:solutions}.

\textbf{Faithfulness}.  Another essential requirement is that $\calE$ covers all and only those concepts that determine the target prediction but not others \citep{marques2023logic, darwiche2023logic}.
To see why, consider the example in \cref{fig:page-two} (right): despite covering the $k = 2$ highest-weight activations (as per \cref{eq:heuristics}), $\calE$ supports predictions for two different classes, {\tt Vesper Sparrow} and {\tt Black-billed Cuckoo}. Since the same activations within $\calE$ are compatible with different predictions, the predicted class cannot be determined by $\calE$ alone, but it must depend on activations \textit{outside} $\calE$. In other words, $\calE$ is insufficient to fully justify either decision.
Unfaithful explanations $\calE$ can mislead users, compromising model assessment and reliable interaction.

As illustrated by the example above, the core requirement for a faithful explanation is that $\calE$ cannot simultaneously back up predictions for different classes \citep{shih2018symbolic, ignatiev2020towards}.
Surprisingly, this aspect has been off the radar in the CBM literature, and the go-to heuristics from \cref{eq:heuristics} have been implicitly assumed to be faithful. We show this is not the case. Following the FXAI literature, explanations satisfying this requirement are called \textit{abductive explanations}\footnote{Aka \textit{sufficient reasons} \citep{shih2018symbolic} and \textit{prime implicant explanations} \citep{ignatiev2019abduction}.} (\AXp{s}) \citep{marques2023logic}, which we extend to CBMs as follows:

\begin{definition}
\label{def:axps}
Given a CBM $(g, f)$, a prediction $(\vAct, \hat y)$, and a set of attainable concept activation vectors $\calAct$, then $\calE \subseteq [b]$ is said to be a \emph{weak \AXp} for $\hat{y}$ with respect to $\calAct$ if and only if
\[
    \textstyle
    \forall \vAct' \in \calAct
    \; . \;
    \big(
        \vAct_\calE' = \vAct_\calE
    \big)
    \; \implies \;
    \big(
        \hat{y} \in \argmax_{y \in [m]} f(\vAct')_y
    \big).
    \label{eq:weak-axp}
\] 
A weak \AXp that is also subset-minimal (s.t. no proper subset $\calE' \subset \calE$ is a weak \AXp), is an \AXp.
\end{definition}

The definition requires \textit{sufficiency} -- an adversary cannot construct any counterexample $\vAct' \in \calAct$ that matches $\vAct$ on all concepts in $\calE$ and yet entails a different class\footnote{If a CBM predicts two classes when the same subset of activations attain exactly the same values, these classes cannot be faithfully explained. This is unlikely to happen for discriminatively trained models.} -- and \textit{subset-minimality} -- all selected concepts should be needed for this to hold.
As shown in \cref{fig:page-two} (right) and by our experiments in \cref{sec:q3}, the heuristics in \cref{eq:heuristics} provide no such guarantees for commonly used values of $k$.

\textbf{Enforcing faithfulness.}  If guarantees are desired, a better option is to resort to \textit{exact algorithms} for computing \AXp{s}.
Concretely, we build on the simple but flexible Algorithm~1 of \citet{marques2023logic}, denoted \OriginalExactAlgorithm (Exact Algorithm).
Starting from the full bottleneck, $\calE = [b]$, it iteratively removes unneeded activations following an arbitrary order (\eg activation index), invoking an \textit{oracle} to check whether the candidate $\calE$ admits a counterexample $\vAct' \in \calAct$.
Doing so can be computationally expensive -- in fact, NP-hard -- in general \citep{shih2018symbolic, ignatiev2019abduction}.  However, the decision function of CBMs is linear and $\calAct$ is a conjunction of linear inequalities; therefore, the oracle boils down to a linear program that can be solved efficiently (see \cref{appx:finding-counterexample}).

We propose a variant of \OriginalExactAlgorithm biased toward more compact explanations.
Our algorithm, denoted \OurExactAlgorithm, grows a candidate $\calE = \emptyset$ into a weak \AXp by adding groups of activations following a heuristic-informed and shrinks the result until it is subset minimal.  This is akin to Algorithm~1 of \citet{soria2026formal}, except \OurExactAlgorithm \textit{i}) computes three \AXp{s} -- one for each heuristic in \cref{eq:heuristics} -- and picks the smallest, and \textit{ii}) employs binary search during the growth stage, and as such requires only a logarithmic number of oracle calls.
As shown in \cref{sec:extra-exp-algorithmic-comparison}, \OurExactAlgorithm's explanations are always \AXp{s}, whereas heuristics explanations often are not, while remaining computationally affordable in practice.
Pseudo-code for all algorithms is listed and discussed in \cref{sec:extra-algorithms}.

\section{The Trade-off Between Faithfulness and Compactness}
\label{sec:faithfulness-vs-compactness}

Our goal is to construct explanations $\calE$ that are both \textit{compact} and \textit{faithful}.  These properties, however, are at odds:  any (weak) \AXp must pin all activations that can be (even adversarially) manipulated to yield a different prediction.\footnote{E.g., a weak \AXp can always be obtained by taking $\calE = [b]$:  since this fixes the whole bottleneck, there exists no counterexample.  However, as we discussed, such an explanation is unusable if $b$ is large.}
Hence, the more leeway a CBM offers to construct counterexamples, the larger (weak) \AXp{s} have to be.
We proceed to make this argument exact.
The next result shows how two architectural properties -- (\textbf{P1}) the structure of the activation space $\calAct$, and (\textbf{P2}) the structure of the inference layer's weights $W$ -- determine whether a given $\calE \subset [b]$ can be a weak \AXp.  We begin from \textit{unbounded} activations $\calAct = (-\infty, \infty)^b$ and analyze the other cases to \cref{sec:proofs-appendix}.

\begin{proposition}[Informal]
\label{prop:label-flip-conditions-main}
\label{prop:no-apx-for-regular-cbms}
    Fix any CBM prediction $(\vAct, \hat y) \in \calAct \times [m]$ and explanation $\calE \subset [b]$, and let $\Deltaj := \vw_j - \vw_{\hat y} \in \bbR^b$ be the delta between the weight vectors of classes $j$ and $\hat y$.
    If (\textbf{P1}) $\calAct = (-\infty, \infty)^b$, and (\textbf{P2}) there exists an activation $i \in [b] \setminus \calE$ and a class $j \neq \hat y$ s.t. $\delta{}w^{(j)}_i \neq 0$, then $\calE$ is \textbf{not} a weak \AXp.
\end{proposition}

The proof can be found in \cref{sec:proofs-appendix}.
Hence, if $\calAct$ is unbounded, fixing the prediction $\hat y$ requires $\calE$ to fix \textit{all} concept activations $j$ that link it to any alternative $y'$ via a non-zero weight difference $\delta{}w^{(j)}_i$.
Together with our earlier observation that $\calE = [b]$ is a (weak) \AXp, this entails that for regular CBMs, which (\textbf{P1}) model activations as raw logits and (\textbf{P2}) employ dense linear layers, \textit{no \AXp smaller than the entire bottleneck exists at all}.\footnote{In the binary case ($m = 2$), the top layer has a single output (the score of the positive class) and a single weight vector $\vw \in \bbR^b$.  This is equivalent to a top layer with two outputs and symmetric per-class weight vectors, that is, $\vw_0 = - \vw_1$.  Hence, the precondition $w_{\hat{y},j} - w_{y',j} \ne 0$ becomes $w_{\hat{y},j} \ne 0$: \cref{prop:no-apx-for-regular-cbms} triggers whenever an independent, unbounded activation $c_j$ not in $\calE$ is associated to a non-zero weight.}  This affects both heuristic and exact algorithms.

\cref{prop:no-apx-for-regular-cbms} suggests -- and \cref{prop:label-flip-conditions} confirms -- that \textbf{P1} and \textbf{P2} indeed control for the existence and size of (weak) \AXp{s}:  intuitively, larger activation spaces $\calAct$ leave more room for manipulating concept activations not fixed by $\calE$, while non-zero learned weights, especially if large, translate these into higher scores for non-predicted classes (see \cref{sec:extra-theory-discussion} for a detailed discussion).
Thus, appropriately constraining $\calAct$ and $W$ can reduce the adversary's ability to construct working counterexamples, meaning (weak) \AXp{s} do not have to cover as many activations to safeguard the prediction.

\subsection{Existing CBMs vs. \cref{prop:no-apx-for-regular-cbms}}
\label{sec:existing-cbms-suck}

Next, we discuss existing CBMs in light of \textbf{P1} and \textbf{P2}, summarizing our findings in \cref{tab:methods}.

\begin{table}[!t]
    \centering
    \caption{\textbf{Representative CBMs} grouped by domain of concept activations $\calAct$ (\textbf{P1}) and structure of the inference layer's weights $W$ (\textbf{P2}).  Constraining $\calAct$ and $W$ can make it more difficult to construct counterexamples that flip the prediction, allowing (weak) \AXp to be smaller. Few CBMs employ probabilistic activations and none group lasso.}
    \scalebox{0.9}{
    \begin{tabular}{cll}
        \toprule
        {\sc Act. $\calAct$}
            & {\sc Wei. $W$}
            & {\sc Representative Models (Not Exhaustive)}
        \\
        \midrule
        \rowcolor{tablegray}
        $(-\infty, \infty)^b$
            & Dense
            & CBM~\smallcitep{koh2020concept}, ProbCBM~\smallcitep{kim2023probabilistic}, MCBM~\smallcitep{almudevar2026there}
        \\
        $(-\infty, \infty)^b$
            & Elastic Net
            & LF-CBM~\smallcitep{oikarinen2023label}, VLG-CBM~\smallcitep{srivastava2024vlgcbm}
        \\
        \rowcolor{tablegray}
        $[\alpha, \infty)^b, \alpha \ge 0$
            & Elastic Net
            & Selective CBM \smallcitep{schrodi2025selective}
        \\
        $[-1, 1]^b$
            & Dense
            & LaBo \smallcitep{yang2023language}
        \\
        \rowcolor{tablegray}
        $\subseteq [0, 1]^b$
            & Dense
            & CB2M \smallcitep{steinmann2024learning}, GlanceNets \smallcitep{marconato2022glancenets}
        \\
        \bottomrule
    \end{tabular}
    }
    \label{tab:methods}
\end{table}

\textbf{Concept activation space.}  Standard CBMs \citep{koh2020concept, shin2023closer, park2026analysis} and extensions \citep{sawada2022concept, wu2022learning, huang2024concept, almudevar2026there} employ independent, unbounded activations, and thus trigger \cref{prop:no-apx-for-regular-cbms}.  Despite the name, probabilistic CBMs do the same \citep{kim2023probabilistic}.
VLM-CBMs with a backbone \citep{oikarinen2023label, srivastava2024vlgcbm} do supervise it with bounded concept annotations, yet feed the inference layer unbounded activations.
Post-hoc CBMs \citep{yuksekgonul2023post} lack a backbone, but they employ unbounded dot products as activations.  All these models leave the adversary unconstrained.
Selective CBMs \citep{schrodi2025selective} pass activations through an input-dependent gating function $\max(0, \vAct - \vo)$ with a learnable bias $\vo$.\footnote{The JumpReLU/TopK-based variants \citep{schrodi2025selective} behave analogously.}  While this prevents \cref{prop:no-apx-for-regular-cbms} from applying, it allows pushing activations not in $\calE$ to $+\infty$, which can flip the label whenever these are associated to positive (rather than non-zero) weight differences $\delta{}w^{(j)}_i$ (see \cref{prop:label-flip-conditions} for a formal statement); as we show in \cref{sec:extra-cbm-vs-selective-cbm}, this inflates faithful explanations in practice.
LaBo \citep{yang2023language} fares best in that its activations are cosine similarities living in $\calAct = [-1, 1]^b$, yet we show in \cref{sec:q2} this does not guarantee (weak) \AXp{s} are compact unless the weights of the inference layer are also appropriately constrained.

\textbf{Inference layer weights.}  Standard CBMs do not regularize $W$ and as such do not encourage learning non-zero differences $\Deltaj$.
Recent alternatives -- like Selective CBMs \citep{schrodi2025selective} and most VLM-CBMs \citep{yuksekgonul2023post, oikarinen2023label, srivastava2024vlgcbm} -- sparsify $W$ through an \textit{elastic net} penalty \citep{zou2005regularization}, thus encouraging predictions to depend on few concepts.  This penalty takes the form $\alpha\norm{W}_1 + (1-\alpha) \norm{W}_2^2/2$, where $\alpha$ mediates between the $L_1$ and $L_2$ terms.
While the elastic net encourages \textit{unstructured} sparsity, it still allows learning \textit{pairs} of entries satisfying $\Deltaj \neq \mathbf{0}$, and, as we show in \cref{sec:q2}, it struggles to produce compact faithful explanations for VLM-CBMs.

\subsection{Beyond \cref{prop:no-apx-for-regular-cbms} with Probabilistic Concepts and Group Lasso}
\label{sec:solutions}

\textbf{Probabilistic modeling of concept activations}.  We propose to constrain $\calAct$ by \textit{modeling concepts probabilistically}, \ie by passing activations of binary (resp. categorical) concepts through sigmoid (resp. softmax) activations.
Besides ensuring activations are bounded to $[0, 1]^b$, doing so allows exploiting mutual exclusivity of categorical concepts.
To see why this matters, imagine encoding a categorical ``color'' concept with three values -- red, green and blue -- as three softmax outputs $\Act_{\tt r}$, $\Act_{\tt g}$, $\Act_{\tt b}$.  Critically, $\Act_{\tt r} + \Act_{\tt g} + \Act_{\tt b} = 1$.  Suppose red is predicted with high probability, say, $\Act_{\tt r} = 0.95$.  If $\calE$ covers $s_{\tt r}$, the adversary is left little room to construct a counterexample $\vAct'$:  at most, it can shift the residual $0.05$ mass between $\Act_{\tt g}$ and $\Act_{\tt b}$, which is unlikely to flip the prediction.
Formally, this shrinks $\calAct$ from $[0, 1]^b$ to a product of per-concept simplices:
the two only match if all concepts are binary.

\textbf{Structured weight sparsity}.  We suggest to constrain the weights $W$ via \textit{group lasso} \citep{yuan2006model}, treating each activation as a separate group.
This penalty takes the form $\sum_i \norm{\vw_i}_2$, where $i$ iterates over activations, so as to actively discourage each of them from contributing to any class.  In turn, this encourages all differences $w_{y',j'} - w_{\hat{y},j'}$ to be exactly zero.

\textbf{Benefits \& Limitations}.  As we will show empirically in \cref{sec:experiments}, both proposed strategies are effective at shrinking the size of (weak) \AXp{s}.
Probabilistic modeling of concepts is not new (see \eg \citep{marconato2022glancenets, steinmann2024learning, tan2024sparsity}, but it has never been linked to the compactness of faithful explanations.
In contrast, despite being a widely known technique, no existing CBM architecture employs group lasso.
As we will see, probabilistic modeling does not come at any accuracy cost, while excessive sparsification \textit{can} lower downstream accuracy (see \cref{sec:q2,sec:q3}) and as such has to be applied with care.
We stress, however, that even our strategies, while helpful, cannot ensure large bottlenecks always admit \textit{small} \AXp{s}.  We will briefly outline other potentially promising remedies in \cref{sec:conclusion}.

\section{Experiments}
\label{sec:experiments}

We address the following research questions:
\textbf{Q1}: Does modeling concepts probabilistically enable more compact \AXp{s}?
\textbf{Q2}: Does group lasso shrink \AXp{s} further?
\textbf{Q3}: How does \OurExactAlgorithm fare compared to the heuristic algorithms in \cref{eq:heuristics}?
Due to space constraints, we leave some experimental details to \cref{sec:expermental-details} and additional results to \cref{sec:extra-experiments}.

\begin{wraptable}[7]{r}{0.575\textwidth}
    \centering
    \vspace{-2em}
    \caption{\textbf{Datasets.} {\sc Cat.} is the \# of categorical concepts.}
    \label{tab:dataset-stats}
    \scalebox{0.8}{
    \begin{tabular}{lrrrrrr}
        \toprule
        {\sc Dataset} & $m$ & $b$ & {\sc Cat.} & {\sc Train} & {\sc Val} & {\sc Test} \\
        \midrule
        \CelebA         &   2 &  39 &  0 & 162{,}770 & 19{,}867 & 19{,}962 \\
        \Derm{}         &   5 &  28 &  7 &       413 &      203 &     395 \\
        \WBC{}          &   5 &  31 & 11 &   6{,}169 &  1{,}030 &  3{,}099 \\
        \CUB            & 200 & 312 & 28 &   5{,}394 &      600 &  5{,}794 \\
        \bottomrule
    \end{tabular}
    }
\end{wraptable}
\textbf{Datasets.}  We evaluate four CBM datasets spanning different number of classes $m$, concept activations $b$, and ratio of binary vs. categorical concepts, cf. \cref{tab:dataset-stats}.
\underline{\CelebA} \citep{liu2015deep} comprises $203$K face images annotated with $40$ binary concepts, split into $163$K training / $20$K validation / $20$K test images. The goal is to predict {\tt sex} from the remaining $39$ concepts.
{\tt \underline{Derm7}p\underline{t}} \citep{kawahara2018seven} includes $1$K dermoscopic skin-lesion cases annotated with diagnostic criteria, one-hot encoded into $28$ activations, and a $5$-class diagnosis label obtained by grouping $20$ raw diagnoses.
\underline{\WBC} \citep{tsutsui2023wbcatt} annotates $10$K white blood cell microscopy images with $11$ categorical morphological attributes, one-hot encoded into $31$ activations, and a $5$-class cell-type label.
\underline{\CUB} \citep{wah2011caltech} involves classifying bird images into one of $200$ classes using $28$ categorical concepts (multi-hot encoded into $312$ activations).
It comprises $5$K training and $5$K test images;  we carve about $10\%$ of the training set for validation.

\textbf{Models.}  For all models, the backbone is a frozen pre-trained image encoder -- OpenCLIP's {\tt ViT-L-14-quickgelu} \citep{ilharco2021openclip} -- followed by a linear concept probe.
The inference layer is a logistic regressor, either dense or sparsified via elastic net (\ENet) or group lasso (\GLasso).
For regular CBMs, we take the concept vocabularies and annotations from the dataset and evaluate logit ($\calAct = (-\infty, +\infty)^b$), ReLU ($\calAct = [0, +\infty)^b$), and probabilistic ($\calAct \subseteq [0, 1]^b$) concept activations.  For the latter, we employ softmax activations for categorical concepts and sigmoids for binary ones, as per \cref{sec:solutions}.
We resort to sequential training, as it achieves competitive performance while reducing the amount of label-specific information affecting the concepts.
For VLM-CBMs, we employ LLM-generated vocabularies from \LaBo \citep{yang2023language}, \LFCBM \citep{oikarinen2023label}, and \VLGCBM \citep{srivastava2024vlgcbm} and compute concept activations as cosine similarities in CLIP's embedding space, as customary \citep{debole2026if}.  Following \cref{sec:existing-cbms-suck}, we ensure $\calAct = [-1, 1]^b$ by employing the obtained activations for inference, like \LaBo.\footnote{The original \LFCBM and \VLGCBM employ a different VLM and a separate backbone; this explains the accuracy drop in our results compared to the original papers. This difference is irrelevant for our purposes.}
Due to space constraints, we defer our analysis of Selective CBMs to \cref{sec:extra-cbm-vs-selective-cbm}.

\textbf{Metrics.}  For each model, we report the test set label and concept $F_1$ score (denoted $F_1^Y$ and $F_1^C$).  The latter is micro-averaged over concept activations when these are independent (\ie logits and ReLUs) and over concepts otherwise (\ie for probabilities).
For each algorithm, we explain a random subset of $1$K test instances; for \Derm we explain all $395$ test points; for the \LaBo\ \CUB vocabulary, which has $10$K entries, we explain $100$ instances.
For each model, \textit{we test faithfulness following \cref{def:axps} using the model's declared activation space $\calAct$}; as we discuss in \cref{sec:extra-activation-spaces}, this guarantees our weak \AXp{s} are robust to all possible counterexamples, regardless of whether are produced by the backbone or by (potentially malicious) human interventions.

\subsection{Q1: Probabilities Shrink Faithful Explanations At No Accuracy Cost}
\label{sec:q1}

We compute \textit{faithfulness-compactness curves} (FCCs) plotting the cumulative percentage of explained instances admitting a (weak) \AXp within a specified size budget $k$:  the faster they grow, the more instances admit a smaller faithful explanation.
In \cref{fig:results-logits-vs-proba}, we report FCCs for dense CBMs:  logit activations above and probabilities below.  Colors represent different algorithms: the {\tt TopK} heuristics in \cref{eq:heuristics} ({\tt Act} in \textcolor{teal}{\textbf{green}}, {\tt Wei} in \textcolor{orange}{\textbf{orange}}, {\tt Prod} in \textcolor{violet}{\textbf{violet}}) and \OurExactAlgorithm (\textcolor{magenta}{\textbf{pink}}).

\begin{wraptable}[9]{r}{0.425\textwidth}
    \centering
    \vspace{-2em}
    \caption{\textbf{Probabilities are competitive with logits} in terms of $F_1$ score. }
    \scalebox{0.8}{
    \begin{tabular}{lrrrr}
        \toprule
        & \multicolumn{2}{c}{\sc $F_1^C$} & \multicolumn{2}{c}{\sc $F_1^Y$} \\
        \cmidrule(lr){2-3} \cmidrule(lr){4-5}
        {\sc Dataset} & {\sc prob} & {\sc logit} & {\sc prob} & {\sc logit} \\
        \midrule
        \CelebA     & $0.760$ & $0.760$ & $0.989$ & $0.991$ \\
        \Derm       & $0.676$ & $0.670$ & $0.653$ & $0.666$ \\
        \WBC        & $0.877$ & $0.872$ & $0.943$ & $0.948$ \\
        \CUB        & $0.796$ & $0.781$ & $0.817$ & $0.832$ \\
        \bottomrule
    \end{tabular}
    }
    \label{tab:f1-logits-vs-proba}
\end{wraptable}
\textbf{Logits are pathological}.  For all algorithms, the curves for logit activations are squashed to zero until the explanation size budget reaches the bottleneck size $b$.
This empirically confirms \cref{prop:no-apx-for-regular-cbms}:  logits leave the adversary unconstrained, meaning that -- since the inference layer is dense -- no (weak) \AXp smaller than the entire bottleneck exists at all.

\begin{figure}[!t]
    \centering
    \begin{tabular}{c}
         \includegraphics[width=0.315\linewidth]{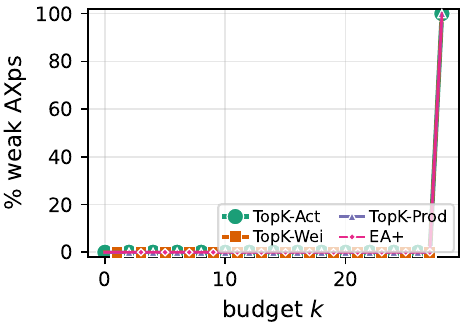}
         \includegraphics[width=0.315\linewidth]{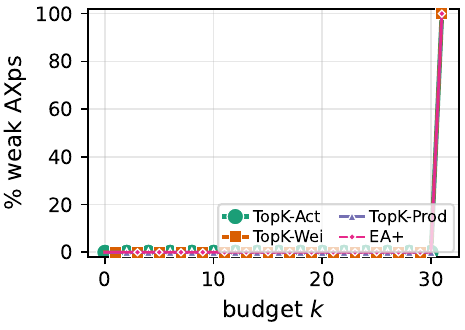}
         \includegraphics[width=0.315\linewidth]{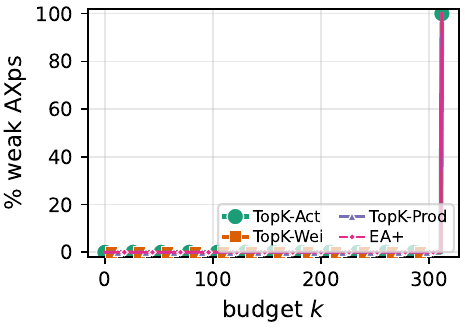}
         \\
         \includegraphics[width=0.315\linewidth]{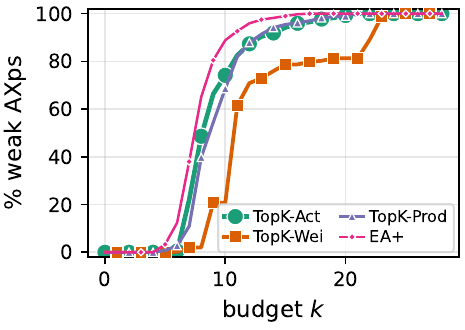}
         \includegraphics[width=0.315\linewidth]{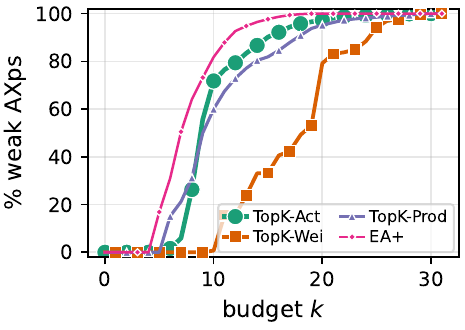}
         \includegraphics[width=0.315\linewidth]{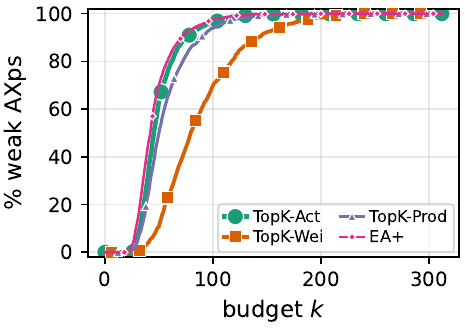}
    \end{tabular}
    \caption{\textbf{Modeling concepts as probabilities shrinks \AXp{s}} on \Derm (left), \WBC (middle), and \CUB (right). Top row: logits. Bottom row: probabilities.  \CelebA is left to \cref{sec:extra-exp-logits-vs-relus-vs-proba}.
    }
    \label{fig:results-logits-vs-proba}
\end{figure}

\textbf{Probabilities shrink faithful explanations of all algorithms without hindering accuracy}.  For probabilistic activations, the FCCs grow much quicker, meaning several explained instances admit smaller (weak) \AXp{s}.
Focusing on \OurExactAlgorithm, the FCC shows that about $80\%$ instances ($Y$-axis) admit a faithful explanation of size $\le 9$ in \Derm (down from $28$ for logits), $\le 10$ in \WBC ($31$ for logits), and $\le 60$ in \CUB ($312$ for logits).  The results for \CelebA in \cref{sec:extra-exp-logits-vs-relus-vs-proba} are analogous.
Comparing the FCCs of the heuristics in the top and bottom rows shows they also require a smaller budget $k$ to find weak \AXp{s}: \TopKAct and \TopKProd fare best, \ie explain compactly and faithfully almost as many instances as \OurExactAlgorithm, while \TopKWei lags behind.
Overall, these results confirm that modeling concepts probabilistically is effective at preventing counterexamples.
Crucially, these benefits come at no cost as CBMs with logit and probabilistic activations perform equally well:  \cref{tab:f1-logits-vs-proba} shows that they attain almost identical $F_1^Y$ and $F_1^C$, with the worst-case gap being $2$ percentage points (pp).

\subsection{Q2: Group Lasso Outperforms Elastic Net}
\label{sec:q2}

Next, we compare \ENet and \GLasso on CBMs with probabilistic activations and VLM-CBMs with cosine similarity activations.
We focus here on the average size of the (weak) \AXp{s} found by the various algorithms,
and report the FCCs for all models and datasets in \cref{sec:extra-exp-sparsity-on-proba,sec:extra-exp-sparsity-on-vlm-cbms}.
Note that cranking up sparsity strength $\lambda$ helps finding faithful explanations (per \PTwo) but also sinks performance.  To ensure a fair comparison, for each sparse CBM we maximize $\lambda$ under the constraint that its $F_1^Y$ is at most $10\%$ worse than the dense CBM's on the validation split, as shown in \cref{tab:f1-proba-dense-vs-enet-vs-glasso}.
We fix \ENet's $\alpha$ to $0.99$, \citet{yuksekgonul2023post} and \citet{oikarinen2023label}'s default.

\begin{wraptable}[13]{r}{0.425\textwidth}
    \centering
    \vspace{-1em}
    \caption{\textbf{\GLasso shrinks faithful explanations at less than $10\%$ $F_1^Y$ cost}.  VLM-CBMs have no $F_1^C$ as they have no ground-truth annotations.}
    \scalebox{0.8}{
    \begin{tabular}{lrrrr}
        \toprule
        & {\sc $F_1^C$} & \multicolumn{3}{c}{\sc $F_1^Y$} \\
        \cmidrule(lr){3-5}
        {\sc Dataset} &  & {\sc dense} & \ENet & \GLasso \\
        \midrule
        \CUB        & $0.796$ & $0.817$ & $0.759$ & $0.756$ \\
        \Derm       & $0.676$ & $0.653$ & $0.648$ & $0.638$ \\
        \WBC        & $0.877$ & $0.943$ & $0.916$ & $0.927$ \\
        \midrule
        \LFCBM      & --      & $0.545$ & $0.589$ & $0.564$ \\
        \VLGCBM     & --      & $0.626$ & $0.632$ & $0.576$ \\
        \LaBo       & --      & $0.835$ & $0.753$ & $0.766$ \\
        \bottomrule
    \end{tabular}
    }
    \label{tab:f1-proba-dense-vs-enet-vs-glasso}
\end{wraptable}
\textbf{\ENet only partially works}.  For CBMs (\cref{tab:q2}, top),
in most cases \ENet shrinks explanation sizes, but in three cases (in \textcolor{red}{red}) it actively regresses:  barely on \Derm ($+1.8$pp for \TopKProd) and substantially on \CUB ($+58.4$pp for \TopKWei and $+68.5$pp for \TopKProd).
This is because at $\alpha = 0.99$, which favors $L_1$ over $L_2$, it tends to learn larger non-zero weights, which allow the adversary to flip the prediction with smaller activations deltas (as larger weights help convert these into comparatively larger changes to the class scores).
\ENet also fares poorly with VLM-CBMs, which have less constrained concept activations: it often fails to shrink explanations (see the \textcolor{red}{red} entries), except for \OurExactAlgorithm (\LFCBM: $-16.6$pp; \VLGCBM: $-30.7$pp; \LaBo: $-52.7$pp).

\textbf{\GLasso performs better and is key for VLM-CBMs.}  \GLasso always yields smaller faithful explanations compared to the dense baseline.  E.g., for \OurExactAlgorithm, the improvement amounts to \Derm: $-22.3$pp; \WBC: $-18.3$pp; \CUB: $-7.7$pp.
It only loses to \ENet by a small margin in two cases: \WBC/\OurExactAlgorithm (explanation size $7.3\%$ vs $8.1\%$) and \CUB/\TopKAct ($13.1\%$ vs $14.3\%$).
Moreover, \cref{tab:q2} (bottom) shows that \GLasso is essential for VLM-CBMs, where it substantially outperforms \ENet (except for \TopKAct, which misbehaves in all cases): with \GLasso, \OurExactAlgorithm explanations take $17.5\%$ of \LFCBM's bottleneck (whereas \ENet requires $82.9\%$ of it), $15.9\%$ of \VLGCBM's (versus $68.7\%$), and only $6.5\%$ of \LaBo's (versus $46.9\%$).

\begin{table}[!t]
    \centering
    \caption{\textbf{\GLasso consistently yields smaller faithful explanations than \ENet}.  Explanation size is reported as a percentage (\ie normalized by bottleneck size $b$) for concision.  \textbf{Bold} indicates best in column and \textcolor{red}{red} no shrinkage w.r.t. the dense baseline.}
    \label{tab:q2}
    \scalebox{0.8}{
    \begin{tabular}{l cccc cccc cccc}
        \toprule
                & \multicolumn{4}{c}{\Derm ($b = 28$)}
                & \multicolumn{4}{c}{\WBC ($b = 31$)}
                & \multicolumn{4}{c}{\CUB ($b = 312$)} \\
        \cmidrule(lr){2-5} \cmidrule(lr){6-9} \cmidrule(lr){10-13}
                & {\sc Act} & {\sc Wei} & {\sc Prod} & \OurExactAlgorithm
                & {\sc Act} & {\sc Wei} & {\sc Prod} & \OurExactAlgorithm
                & {\sc Act} & {\sc Wei} & {\sc Prod} & \OurExactAlgorithm \\
        \midrule
        dense   
                & $34.1$ & $47.3$ & $35.1$ & $29.6$
                & $34.0$ & $57.6$ & $35.6$ & $26.3$
                & $16.6$ & $28.6$ & $18.5$ & $15.4$ \\
        \ENet   
                & $27.7$ & $36.9$ & $\textcolor{red}{36.9}$ & $9.5$
                & $29.9$ & $31.0$ & $30.9$ & $\textbf{7.3}$
                & $\textbf{13.1}$ & $\textcolor{red}{87.0}$ & $\textcolor{red}{87.0}$ & $10.2$ \\
        \GLasso 
                & $\textbf{27.4}$ & $\textbf{8.1}$  & $\textbf{7.8}$  & $\textbf{7.3}$
                & $\textbf{28.1}$ & $\textbf{11.7}$ & $\textbf{13.6}$ & $8.1$
                & $14.3$ & $\textbf{9.8}$  & $\textbf{9.4}$  & $\textbf{7.7}$ \\
        \midrule
        \midrule
                & \multicolumn{4}{c}{\LFCBM ($b = 202$)}
                & \multicolumn{4}{c}{\VLGCBM ($b = 419$)}
                & \multicolumn{4}{c}{\LaBo ($b = 10{,}000$)} \\
                & {\sc Act} & {\sc Wei} & {\sc Prod} & \OurExactAlgorithm
                & {\sc Act} & {\sc Wei} & {\sc Prod} & \OurExactAlgorithm
                & {\sc Act} & {\sc Wei} & {\sc Prod} & \OurExactAlgorithm \\
        \midrule
        dense   
                & $100.0$ & $99.9$ & $99.9$ & $99.5$
                & $99.9$  & $99.8$ & $99.8$ & $99.4$
                & $99.8$  & $99.6$ & $99.6$ & $99.6$ \\
        \ENet   
                & $\textcolor{red}{100.0}$ & $\textcolor{red}{99.9}$ & $\textcolor{red}{99.9}$ & $82.9$
                & $\textcolor{red}{100.0}$ & $\textcolor{red}{99.9}$ & $\textcolor{red}{99.9}$ & $68.7$
                & $\textcolor{red}{99.9}$  & $99.1$ & $99.1$ & $46.9$ \\
        \GLasso 
                & $\textcolor{red}{99.7}$  & $\textbf{19.3}$ & $\textbf{19.4}$ & $\textbf{17.5}$
                & $\textcolor{red}{99.8}$  & $\textbf{17.1}$ & $\textbf{17.1}$ & $\textbf{15.9}$
                & $\textcolor{red}{99.9}$  & $\textbf{6.6}$  & $\textbf{6.7}$  & $\textbf{6.5}$ \\
        \bottomrule
    \end{tabular}
    }
\end{table}

\subsection{Q3: \OurExactAlgorithm finds smaller faithful explanations}
\label{sec:q3}

Finally, we assess how the various algorithms fare in recovering compact (weak) \AXp{s}, focusing on properly constrained CBMs to enable meaningful comparison.
Additionally, we report an in-depth analysis of runtimes, number of oracle calls, and subset-minimality in \cref{sec:extra-exp-algorithmic-comparison}.

\textbf{\OurExactAlgorithm finds \AXp{s} at a lower budget $k$}.  While \OurExactAlgorithm finds $k$ automatically, the {\tt TopK} heuristics take it as input:  if too small, they \textit{cannot} find any weak \AXp. 
This is what the FCCs in \cref{fig:results-logits-vs-proba} (bottom row) show for CBMs with probabilistic concepts:\textit{  setting $k$ too low yields no weak \AXp at all}.
Moreover, compared to \OurExactAlgorithm, heuristics need a larger $k$ to find a weak \AXp.  This is expected: \OurExactAlgorithm returns the smallest of three \AXp{s}, each computed by refining a weak \AXp obtained by ranking activations using weights, activations, or their product, cf. \cref{sec:extra-algorithms}.
Among heuristics, \TopKAct and \TopKProd work best.  This is due to modeling categorical concepts as probabilities:  as discussed in \cref{sec:solutions}, they pin the largest activations first, constraining the other activations via mutual exclusivity in turn.
This advantage disappears when all concepts are binary, as shown on \CelebA in \cref{sec:extra-exp-logits-vs-relus-vs-proba} and by VLM-CBMs.
\OurExactAlgorithm also yields smaller weak \AXp{s} for VLM-CBMs, as shown in \cref{tab:q2}.
The gap is substantial for \ENet-sparsified models:  the best heuristic yields explanations with average size $> 99\%$ of the bottleneck across vocabularies, while \OurExactAlgorithm attains $82.9\%$ (\LFCBM), $68.7\%$ (\VLGCBM) and $46.9\%$ (\LaBo).
On \GLasso, \OurExactAlgorithm still wins but with a smaller margin, attaining $17.5\%$, $15.9\%$ and $6.5\%$ versus $19.3\%$, $17.1\%$ and $6.6\%$ for the runner-up.

\section{Related Work}
\label{sec:related-work}

\textbf{Formal XAI} sets itself apart from classic XAI \citep{guidotti2018survey} by seeking explanations with formal guarantees \citep{marques2023logic, darwiche2023logic}.
Most works focus on discrete features, and as such do not apply to CBMs with continuous concepts, although a handful of CBM variants do employ discrete concepts, \eg \citep{havasi2022addressing, vandenhirtz2024stochastic, knab2026s}; we plan to investigate their faithfulness in future work.
We build on the well-known \OriginalExactAlgorithm algorithm due to its flexibility and good empirical performance.
More refined algorithms for computing \textit{size-minimal} \AXp{s} of linear classifiers with polynomial delay \citep{marques2020explaining, cooper2023tractability, fernandes2026concisely} are restricted to binary tasks and cannot replace \OurExactAlgorithm in our experiments.
Our theoretical contribution is novel.  Compatibly with it, the algorithms of \citet{marques2020explaining} require continuous features to be bounded.  Our work, however, studies not only existence of \AXp{s} but also modeling choices (beyond boundedness) that control their \textit{size}.

Three papers lie closest to our work.
\citet{srivastava2024vlgcbm} show that top-$5$ heuristic explanations of VLM-CBM decisions often allow flipping predictions;  our results systematize their observation.
\citet{soria2026formal} introduce FXAI algorithms for computing \AXp{s} of prototype neural nets, where concept activations encode similarities to learned prototypical embeddings.  Our theory backs up their findings but also allows us to identify architectural choices that facilitate the existence of more compact (weak) \AXp{s}.  Their Algorithm~1 is close to \OurExactAlgorithm except it ranks activations by absolute value only, which performs poorly for independent activations (see \CelebA in \cref{sec:extra-exp-logits-vs-relus-vs-proba}) and for VLM-CBMs (see \cref{sec:q3}).
Finally, \citet{paul2024formal} design an FXAI algorithm for neuro-symbolic models that, similarly to CBMs, pair concept extraction with (logic) inference \citep{marconato2026symbol}.  Whereas they seek hierarchical \AXp{s} (covering both inputs and concepts) robust to local ($\varepsilon$-ball) perturbations, we construct concept-level \AXp{s} robust to all counterexamples in $\calAct$.

\textbf{Transparency of CBMs.}  Most critiques of CBMs study human-alignment of the learned concepts.
Some works show that, without concept supervision, CBMs can learn arbitrarily misaligned concepts \citep{bortolotti2025shortcuts}; others look at the impact of replacing expert supervision with VLM annotations
\citep{kazmierczak2026enhancing, debole2026if}; others still show that, even when concepts are supervised, CBMs can push accuracy by encoding spurious information into the concepts \citep{mahinpei2021promises, margeloiu2021concept, raman2023concept} and propose several remedies \citep{havasi2022addressing, marconato2022glancenets, huang2024concept, fokkema2026sample, almudevar2026there} (but see \citet{zarlenga2026defense} for counterarguments).
While concept alignment is sometimes referred to as ``concept faithfulness'' \citep{schmalwasser2026faithfulness}, none of these works consider the faithfulness of CBM \textit{explanations}.

\section{Conclusion}
\label{sec:conclusion}

We study the faithfulness of CBM explanations and adapt ideas and techniques from FXAI to these models.
Our work shows that many existing CBMs do not admit compact faithful explanations;
that probabilistic modeling and structured sparsity enable retrieving substantially more compact (weak) \AXp{s} for regular and VLM-based CBMs;
and that, while existing heuristic strategies for computing compact CBM explanations often fail to identify (even weak) \AXp{s}, our exact algorithm does.
At a high level, our work suggests that (large) concept bottlenecks should not be na\"ively equated with faithful interpretability.
In future work, we will study how to shrink faithful explanations further.
A first direction involves squeezing \citep{schrodi2025selective, yan2023learning} and enlarging \citep{feng2026bayesian} bottlenecks \textit{on demand} while ensuring faithfulness.
A second direction involves relaxing the notion of faithfulness, \eg by adapting \textit{probabilistic} variants \citep{wang2021probabilistic, azzolin2025reconsidering, subercaseaux2025probabilistic} that require the prediction holds with high probability.
Finally, while we require explanations to be robust to both natural and (possibly out-of-distribution) human interventions, we plan to study how to shrink \AXp{s} by considering only counterexamples achievable by the backbone, leveraging approximations from the neural net verification literature \citep{singh2019abstract, zhang2018efficient} to keep \OurExactAlgorithm tractable.  As a side benefit, this would automatically account for dependencies between input variables, which FXAI algorithms (including ours) typically ignore.

\clearpage
\newpage

\subsection*{AI use statement}

We used generative AI tools for writing a first version of the code, for support when deriving the proofs in \cref{sec:proofs-appendix}, for scanning the literature in search of relevant publications we originally missed, and for reviewing the manuscript before submission.  We have not used generative AI tools for other tasks with required disclosure.
The code was manually reviewed by the authors.  The proofs were manually written and carefully checked by the authors.  The figures were made by the authors.  The manuscript was written by the authors.  The bibliography was manually compiled by the authors using bibentries sourced from Google Scholar.
We take responsibility for the final content of this work, including text, claims or artifacts produced with the aid of generative AI.

\subsubsection*{Author Contributions}

The core ideas were born from a chat between AV and ST while the latter was visiting the University of Edinburgh in 2024.  The idea evolved over time through several discussions with SA.  In the summer of 2026, ST made a push toward translating ideas into a paper, designing the experiments and a basic theoretical analysis, and leveraging generative AI to write the necessary code.  AV, SA and EM provided invaluable ideas and feedback throughout the process.  Meanwhile, EM developed and integrated a more in-depth theoretical analysis.  All authors contributed to the writing of the paper and approved the current manuscript.

\subsection*{Reproducibility statement}

Experimental details are provided in \cref{sec:experiments} and \cref{sec:extra-experiments}.
%

\subsubsection*{Acknowledgments}

The authors are grateful to Dario Pesenti for useful discussion.
Funded by the European Union. Views and opinions expressed are however those of the author(s) only and do not necessarily reflect those of the European Union, the MSCA-PF programme, or the European Health and Digital Executive Agency (HaDEA). Neither the European Union nor the granting authority can be held responsible for them. Grant Agreement no. 101120763 - TANGO.
EM is funded by the Horizon Europe/Marie Sklodowska-Curie under Grant Agreement No. 101282115 - AI-CURE.
Antonio Vergari is supported by the ``UNREAL: Unified Reasoning Layer for Trustworthy ML'' project (EP/Y023838/1) selected by the ERC and funded by UKRI EPSRC.

\bibliography{references}
\bibliographystyle{unsrtnat}

\clearpage
\newpage
\appendix

\clearpage
\section{Algorithmic Details}
\label{sec:extra-algorithms}

The algorithms we evaluate are the following:
\textit{i}) \TopKAct, \TopKWei, and \TopKProd, three heuristic strategies that require a target explanation size $k$ as input;
\textit{ii}) \OriginalExactAlgorithm, a general exact algorithm, appearing as Algorithm~1 in \citet{marques2023logic}, which does \textit{not} require $k$ as input; and
\textit{iii}) \OurExactAlgorithm, our improved \OriginalExactAlgorithm.
All pseudo-code assumes access to a CBM $(f, g)$ with a linear inference layer and to a prediction $(\vAct, \hat y)$ to be explained;  we leave these implicit for brevity.

\subsection{Finding a Counterexample}
\label{appx:finding-counterexample}

Most algorithms rely on a procedure $\textsc{Oracle}(\calE, \calAct)$ for checking whether a given explanation $\calE$ admits constructing a counter-example $\vAct' \in \calAct$.
Given $(\vAct, \hat y)$, a vector $\vAct'$ is a valid counterexample if it matches $\vAct$ on all activations covered by $\calE$ but yields a different prediction $y' \ne \hat y$, that is:
\[
    \textstyle
    \big(
        \vAct_\calE' = \vAct_\calE
    \big)
    \; \land \;
    \big(
        \hat{y} \not\in \argmax_{y \in [m]} f(\vAct')_y
    \big).
    \label{eq:whatever}
\]
We are dealing with \textit{linear} inference layers of the form $f(\vAct) = W \vAct + \vw_0$.  Hence, letting $\vw_y \in \bbR^b$ be the vector of inference layer weights $W$ associated to class $y \in [m]$, and $w_{y,0}$ be the corresponding bias term, the second term in \cref{eq:whatever} translates to:
\[
    \inner{\vw_{y'}}{\vAct'} + \vw_{y', 0} > \inner{\vw_{\hat y}}{\vAct'} + \vw_{\hat y, 0}.
    \label{eq:whatever2}
\]
We ignore follow-up monotonic transformations, \eg sigmoid and softmax activations, as they do not alter the ranking between classes.
In order to find a counterexample, the oracle has to find a class $y' \ne \hat y$ out of $m - 1$ possible alternatives and a vector $\vAct' \in \calAct$ that satisfies \cref{eq:whatever2}.  This can be readily done by solving, for each candidate $y'$, the following optimization problem:
\begin{align}
    \max_{\vAct'} \quad
        & \inner{\vw_{y'} - \vw_{\hat y}}{\vAct'} + (\vw_{y', 0} - \vw_{\hat y, 0})
    \\
    \mathrm{s.t.} \quad
        & \vAct'_\calE = \vAct_\calE \land \vAct' \in \calAct
\end{align}
The objective is \textit{linear} in $\vAct'$.  At the same time, in our setting the feasible set $\calAct$ can be encoded as a (conjunction of) \textit{linear constraints}.
For cosine similarities ($\calAct = [-1, 1]^b$), ReLU activations ($\calAct = [0, +\infty)^b$), and raw logits ($\calAct = (-\infty, +\infty)^b$), the constraint is a conjunction of axis-aligned open or closed bounds $a_j \le s'_j \le b_j$.
For probabilistic activations ($\calAct \equiv \vdelta{}\vw \subseteq [0, 1]^b$, where $\vdelta{}\vw$ is a product of simplices) it is a conjunction of per-concept linear constraints $\sum_j \Act_{i,j} = 1$, where $i$ runs over concepts and $j$ over activations encoding that concept, and non-negativity constraints $s_j \ge 0$ running over all activations.
Hence, finding a counterexample boils down to a linear program (LP).

\textbf{In short}: the oracle is a collection of $m - 1$ LPs.  If the LP for class $y'$ admits a positive objective, its solution $\vAct'$ is a valid counterexample for $\calE$ within $\calAct$, and the oracle returns it.  Otherwise, if no LP yields a positive objective, $\calE$ admits no counterexample and the oracle returns {\tt false}.

\subsection{The Heuristic Algorithms}
\label{sec:extra-algorithms-heuristics}

\begin{wrapfigure}[8]{r}{0.45\linewidth}
    \vspace{-2.25em}
    \begin{minipage}[t]{\linewidth}
    \begin{algorithm}[H]
        \caption{$\textsc{TopK}^\star(\calAct, b, h)$}
        \label{alg:extra-topk-star}
        \begin{algorithmic}[1]
            \State $\pi \gets [b]$ decreasing by $h$'s criterion
            \For{$k = 1, \ldots, b$}
                \If{not $\textsc{Oracle}(\pi_{1:k}, \calAct)$}
                    \State \Return $\pi_{1:k}$
                \EndIf
            \EndFor
            \State \Return $[b]$
        \end{algorithmic}
    \end{algorithm}
    \end{minipage}
\end{wrapfigure}
These are the three strategies introduced in \cref{eq:heuristics}.  Given a decision $(\vAct, \hat y)$ and a budget $k$, they sort the $b$ concept activations by either largest activation (\TopKAct), largest absolute weight (\TopKWei), or largest absolute score (activation times weight) (\TopKProd), and construct an explanation $\calE$ by taking the top $k$ activations.
This procedure, however, does not guarantee the output is an \AXp (faithful and subset-minimal) or even a weak \AXp (faithful but not subset-minimal).  Crucially, \textit{the likelihood of finding a (weak) \AXp increases with the budget $k$, making a fair comparison with \OurExactAlgorithm difficult, as the latter does not require to fix $k$}.

For this reason, rather than fixing $k$ \textit{a priori}, we compute it dynamically by taking the smallest $k$ such that the resulting explanation a (weak) \AXp{s}.  This is precisely what {\tt TopK$^*$} (\cref{alg:extra-topk-star}) does:  it first ranks the $b$ concept activations according to value, absolute weight, or absolute score -- depending on the choice of heuristic $h$ -- and then checks prefixes of increasing size $k$ until it finds a (weak) \AXp.  If no proper prefix is faithful, it returns the entire bottleneck.  This ensures the resulting explanation is a weak \AXp at the cost of at most $b$ calls to {\sc Oracle}.

\clearpage
\subsection{The \OriginalExactAlgorithm algorithm}
\label{sec:extra-algorithms-ea}

\begin{wrapfigure}[7]{r}{0.45\linewidth}
    \vspace{-2.25em}
    \begin{minipage}[t]{\linewidth}
        \begin{algorithm}[H]
        \caption{$\textsc{Shrink}(\calAct, \calE)$}
        \label{alg:extra-shrink}
        \begin{algorithmic}[1]
            \For{$j \in \calE$}
                \If{not $\textsc{Oracle}(\calE \setminus \{j\}, \calAct)$}
                    \State $\calE \gets \calE \setminus \{j\}$
                \EndIf
            \EndFor
            \State \Return $\calE$
        \end{algorithmic}
        \end{algorithm}
    \end{minipage}
\end{wrapfigure}
This algorithm is a direct application of the {\sc Shrink} procedure in \cref{alg:extra-shrink}.  The procedure takes a weak \AXp $\calE \subseteq [b]$ as input and iterates over all indices $j \in \calE$ in some order (in practice, lowest to highest, for simplicity):  for each $j$, it checks whether $\calE$ remains a weak \AXp{} -- \ie whether it still admits no counterexamples, using {\sc Oracle} -- even if $j$ is removed.  In this case, it removes $j$, otherwise it keeps it.  Note that \textit{the output is guaranteed to be a subset-minimal weak \AXp}, that is, an \AXp.
The \OriginalExactAlgorithm algorithm simply invokes {\sc Shrink} on the full bottleneck $\calE = [b]$, which is a weak \AXp by construction.  The oracle is called exactly $b$ times.

\subsection{The \OurExactAlgorithm algorithm}
\label{sec:extra-algorithms-ea-plus}

\begin{wrapfigure}[19]{r}{0.45\linewidth}
    \vspace{-2.25em}
    \begin{minipage}[t]{\linewidth}
    \begin{algorithm}[H]
        \caption{$\OurExactAlgorithm(\calAct, b)$}
        \label{alg:extra-ea-plus}
        \begin{algorithmic}[1]
        \For{each heuristic $h$}
            \State $\calE_h \gets \textsc{TopK}^\star_{bs}(\calAct, b, h)$
            \State $\calE_h \gets \textsc{Shrink}(\calAct, \calE_h)$
        \EndFor
        \State \Return $\argmin_h |\calE_h|$
        \end{algorithmic}
    \end{algorithm}

    \begin{algorithm}[H]
        \caption{$\textsc{TopK}^\star_{bs}(\calAct, b, h)$}
        \label{alg:extra-topk-star-bs}
        \begin{algorithmic}[1]
            \State $\pi \gets [b]$ decreasing by $h$'s criterion
            \State $\alpha \gets 0$; \; $\beta \gets b$
            \While{$\alpha < \beta$}
                \State $\mu \gets \lfloor \frac{1}{2} (\alpha + \beta) \rfloor$
                \If{not $\textsc{Oracle}(\pi_{1:\mu}, \calAct)$}
                    \State $\beta \gets \mu$
                \Else
                    \State $\alpha \gets \mu + 1$
                \EndIf
            \EndWhile
            \State \Return $\pi_{1:\alpha}$
        \end{algorithmic}
    \end{algorithm}
    \end{minipage}
\end{wrapfigure}
\OriginalExactAlgorithm ignores instance-specific information (\eg the values and weights of the concept activations) and in doing so may return an overly large \AXp.  At the same time, it requires $b$ oracle calls, which can be problematic for large bottlenecks.

Our algorithm, \OurExactAlgorithm, implements two simple modifications that aim to lessen both issues, see \cref{alg:extra-ea-plus}.  It invokes {\tt TopK$^*$} three times, once for each heuristic $h$, obtaining three weak \AXp $\calE_h$.  Then it applies the {\sc Shrink} procedure for removing possibly useless activations to each of them.  Out of the three resulting \AXp{s}, it returns the smallest.

Moreover, rather than invoking \cref{alg:extra-topk-star}, which costs $b$ oracle calls, it employs a more efficient variant that employs binary search instead, listed in \cref{alg:extra-topk-star-bs}.  The goal is to find the shortest prefix of the bottleneck -- sorted according to the given heuristic $h$ -- that is faithful.  Rather than scanning prefixes linearly as \cref{alg:extra-topk-star} does, it employs binary search: in each iteration, it checks whether the prefix delimited by $(1, \frac{1}{2}\lfloor\alpha + \beta\rfloor)$ admits a counterexample.  If it does not, it is faithful, so the algorithm halves it (by decreasing $\beta$), and repeats.  If it does, it is \textit{not} faithful, so the algorithm enlarges it (by increasing $\alpha$), and repeats.  This modification reduces the number of oracle calls to $O(\log_2 b)$.  This offsets the total cost of having to call the procedure for three times total, once per heuristic.

\clearpage
\newpage
\section{Conditions for (In)Existence of Weak \AXp{s}}
\label{sec:proofs-appendix}

Recall that an explanation is a weak \AXp insofar as it does not allow to \textit{flip} the CBM's prediction by manipulating the concept activations it does not cover.  In this section, we make explicit the conditions for flipping the label of an arbitrary inference layer weights $W$ for counterexamples taken from an activation space of the form
\[
    \textstyle
    \calAct = \bigtimes_{i=1}^b [\alpha_i, \beta_i].
\]
Here, $\alpha_i, \beta_i \in \bbR$ are arbitrary (finite or infinite) per-activation constants  with $\alpha_i < \beta_i$.  This will be our departure point for discussing four different continuous activation spaces $\calAct$ most common in the CBM literature (as discussed in \cref{sec:existing-cbms-suck}), namely:
\begin{itemize}[leftmargin=1.5em]

    \item \textbf{Unbounded}: $\calAct = (-\infty, \infty)^b$, \eg when concept activations are unconstrained logits, as in many CBM architectures (see \cref{tab:methods}).

    \item \textbf{Bounded-from-below}: $\calAct = [\alpha, \infty)^b$, \eg as in Selective CBMs \citep{barker2023selective} where activations are only retained if they exceed a threshold $\alpha \geq 0$. 

    \item \textbf{Bounded}: $\calAct = [\alpha, \beta]^b$, \eg in the case of independent probabilistic activations, in which case $\alpha = 0$, $\beta = 1$.

    \item \textbf{Bounded-on-simplices}: $\calAct \subseteq [0, 1]^b$, a product of per-concept simplices, representing the case of probabilistic modeling with (binary and) categorical concepts.
    
\end{itemize}

\textbf{Notation}.  We focus on concept activations $\vAct \in \calAct$ associated to an input $\vx \in \calX$.  Recall that the inference layer is a linear function $f: \calAct \to \bbR^m$, \ie
\[
    f(\vAct) = W \vAct + \vw_0
    \label{eq:label-logits}
\]
where $W \in \bbR^{m \times b}$ and $\vw_0 \in \bbR^m$, and the prediction is  
\[
    \hat y := \argmax_{y \in [m]} f_y ( \vAct) \, ,
    \quad \text{where} \quad
    f_y (\vAct) = \vw_y^\top \vAct + w_{0,y} \, .
\]
For any class $y' \in [m]$, we use $\vw_{y'} \in \bbR^b$ to indicate its associated weights.
Given a non-predicted class $j$ and a prediction $\hat y$, we denote the deltas between their associated weights and biases as:
\[
    \Deltaj := \vw_j - \vw_{\hat y} \, ,
    \quad \text{and} \quad
    \delta w^{(j)}_0 := w_{0,j} - w_{0, \hat y}
\]
and the difference between their logits as 
\[
    \logdiffj
        := f_j(\vAct) - f_{\hat y}( \vAct)
        = \Deltaj \cdot \vAct + \delta w^{(j)}_0 \, .
\]
Finally, for any real $a \in \bbR$, we write
\[
    a^+ := \max (0, a), \quad 
    a^- := \max(0, -a) \, .
\]

\subsection{The Independent Concept Activations Case}

Throughout, we will use the shorthands $\bar{\calE} := [b] \setminus \calE$ and $\calAct_\calE(\vAct) = \{\vAct' \in \calAct : \vAct'_\calE = \vAct_\calE  \}$.  First, we identify general conditions for label flipping in a bounded asymmetric domain when fixing a subset $\calE \subset [b]$ of concept activations.  Our lemma follows on the footsteps of \citet{marques2020explaining}.

\begin{lemma}
    \label{lemma:label-flip-asymmetric-bounded-domain}
    Let $\calAct = \bigtimes_{i=1}^b[\alpha_i, \beta_i]$.  Take any $\vAct \in \calAct$, a prediction $\hat{y} \in \argmax_{y} f_y(\vAct)$, a non-predicted class $j \in [m]$, and a subset $\calE \subset [b]$.  Let
    $
        \mu_j := \max_{\vAct^* \in \calAct_\calE(\vAct)} \delta\mu_j (\vAct^*)
    $
    be the largest logit difference between classes $j$ and $\hat y$ achievable by manipulating concept activations outside of $\calE$.  This quantity can be rewritten as:
    \[
        \mu_j = \logdiffj + \sum_{i \in \bar\calE} \left[
        (\delta{}w^{(j)}_i)^+ (\beta_i - s_i) + (\delta{}w^{(j)}_i)^- (s_i - \alpha_i)
        \right] \, .
    \]
\end{lemma}

\begin{proof}
    Recall that the difference in logits for activation $j$ is given by
    \[
        \logdiffj = \Deltaj \cdot \vAct + \delta w^{(j)}_0 \, .
    \]
    We expand the dot product as follows
    \begin{align}
        \logdiffj 
            &= \sum_{i \in \calE} \delta{}w^{(j)}_i s_i
                + \sum_{i \in \bar\calE} \delta{}w^{(j)}_i s_i
                + \delta w^{(j)}_0 \, .
                \label{eq:ema-one}
    \end{align}
    Now, take any non-fixed concept activation $i \in \bar\calE$.  Following \citep{marques2023logic}, we maximize the logit of class $j$ by setting $s_i$ to $\beta_i$ if $\delta{}w^{(j)}_i$ is positive and to $\alpha_i$ if $\delta{}w^{(j)}_i$ is negative for each $i \in \bar\calE$.
    Plugging this into \cref{eq:ema-one} and maximizing shows that $\mu_j$, the maximum difference in logits attainable for class $j$, is:
    \begin{align}
        \max_{\vAct^* \in \calAct_{\calE}(\vAct)} \delta\mu_j (\vAct^*)
            &= \sum_{i \in \calE} \delta{}w^{(j)}_i s_i + \sum_{i \in \bar\calE} \left[
                (\delta{}w^{(j)}_i)^+ \beta_i - (\delta{}w^{(j)}_i)^- \alpha_i
            \right] + \delta w^{(j)}_0
        \\
            &= \underbrace{
                \sum_{i \in \calE} \delta{}w^{(j)}_i s_i + \sum_{i \in \bar\calE} \delta{}w^{(j)}_i s_i + \delta w^{(j)}_0
            }_{\logdiffj} \notag
        \\
            & \qquad \qquad \qquad \qquad \qquad
                - \sum_{i \in \bar\calE} \delta{}w^{(j)}_i s_i + \sum_{i \in \bar\calE} \left[
                    (\delta{}w^{(j)}_i)^+ \beta_i - (\delta{}w^{(j)}_i)^- \alpha_i 
                \right]
        \\
        &= \logdiffj +
        \sum_{i \in \bar\calE} \left[ (\delta{}w^{(j)}_i)^+ (\beta_i - s_i) + (\delta{}w^{(j)}_i)^- (s_i - \alpha_i) 
        \right]
        \label{eq:expression-of-mj}
    \end{align}
    This proves the claim.
\end{proof}

The reason we care about $\mu_j$ is that it allows us to rewrite the condition under which a label-flipping counterexample exists as follows:  there must exist a class $j \ne \hat y$ such that $\mu_j > 0$, \ie that attains a larger logit than the prediction $\hat y$.  Equivalently, we can write $\max_{j \ne \hat y} \mu_j > 0$.
Next, we exploit \cref{lemma:label-flip-asymmetric-bounded-domain} to identify explicity conditions under which different choices $\calAct$ admit (weak) \AXp{s} smaller than the bottleneck:

\begin{proposition}[Extended \cref{prop:no-apx-for-regular-cbms}]
\label{prop:label-flip-conditions}
    For any CBM prediction $(\vAct, \hat y) \in \calAct \times [m]$, an explanation $\calE \subset [b]$ smaller than the entire bottleneck \textbf{cannot} be a weak \AXp unless: 
    \begin{enumerate}[leftmargin=1.5em]

        \item \textbf{Unbounded} ($\calAct = (-\infty, \infty)^b$)
        There exists no concept activation $i \in \bar{\calE}$ and alternative class $j \neq \hat y$ such that $\delta{}w^{(j)}_i \neq 0$.

        \item \textbf{Bounded from below}: ($\calAct = [\alpha, \infty)^b$.)
        There exists no concept activation $i \in \bar\calE$ and alternative class $j \neq \hat y$ such that $\delta{}w^{(j)}_i > 0$.

        \item \textbf{Bounded symmetric}: ($\calAct = [\alpha, \beta]^b$) There exists no alternative class $j \neq \hat y$ such that $\max_{j \neq \hat y} \mu_j > 0$, where $\mu_j$ can be written as 
        \[
            \mu_j = \delta\mu_j(\bar \vAct) +\frac{\beta - \alpha}{2}  
            \lVert 
            \Deltaj_{\bar{\calE}}
            \lVert_1 \, ;
        \]
        and $\bar \vAct$ is such that $\bar \vAct_{\bar{\calE}} = \frac{\alpha + \beta}{2}  \,\mathbf{1}_{\bar{\calE}}$ is the mid-point between $\alpha$ and $\beta$ for all coordinates in $\bar{\calE}$. 

        \item \textbf{Bounded asymmetric}:
        ($\calAct = \bigtimes_{i=1}^b [\alpha_i, \beta_i]$).  There exists no alternative class $j \ne \hat y$ such that $\max_{j \neq \hat y} \mu_j > 0$, where $\mu_j$ can be written as 
        \[
            \mu_j = \logdiffj + \sum_{i \in \bar{\calE}} \left[
                (\delta{}w^{(j)}_i)^+ (\beta_i - s_i) + 
                (\delta{}w^{(j)}_i)^- (s_i - \alpha_i)
            \right] \, ;
            \label{eq:ema-two}
        \]
        
    \end{enumerate} 
    
\end{proposition}

\begin{proof}
    Recall that finding the conditions for $\calE$ to be an \AXp{} is equivalent to find conditions for which a label flip cannot occur.  Since cases satisfy the conditions of \cref{lemma:label-flip-asymmetric-bounded-domain}, and we will exploit it below.  We begin from the most general case.

    \textbf{Bounded asymmetric}.  \cref{eq:ema-two} follows directly from the lemma.  An adversary can flip the label if and only if there exist a counterexample $\vAct^* \in \calAct_\calE (\vAct)$ such that $\argmax_j f_j(\vAct^*) \neq \argmax_j f_j(\vAct)$, or, equivalently, $\max_{j \neq \hat y} \mu_j > 0 $.
    If $\mu_j > 0$, then $\max_{\vAct^* \in \calAct_{\calE}(\vAct)} f_j ( \vAct^*) - f_{\hat y} (\vAct^*) > 0$, that is, $\vAct^*$ ensures class $j$ attains a larger logit than the prediction $\hat y$.  Conversely, if $\max_j \mu_j \leq 0$, no $\vAct^*$ can induce a label flip.

    \textbf{Bounded symmetric}.  This is a corollary of \cref{lemma:label-flip-asymmetric-bounded-domain}, as the symmetric domain $[\alpha, \beta]^b$ is a special case of the asymmetric case. 
    Denoting with $\bar \vAct \in \calAct_\calE(\vs)$ the activation with $\bar \vAct_{\bar\calE} = \frac{\alpha + \beta}{2} \mathbf{1}_{\bar\calE}$, we can write
    \begin{align}
        \max_{\vs^* \in \calAct_\calE(\vs)} \delta \mu_j(\vAct^*) &= \sum_{i \in \calE} \delta{}w^{(j)}_i s_i + \sum_{i \in \bar{\calE}} \max \big( 
        \delta{}w^{(j)}_i \beta, \delta{}w^{(j)}_i \alpha
        \big)  + \delta w^{(j)}_0 \\
        &= \underbrace{\sum_{i \in \calE} \delta{}w^{(j)}_i s_i + \sum_{i \in \bar{\calE}} \delta{}w^{(j)}_i \bar s_i +\delta w^{(j)}_0}_{\delta\mu_j(\bar \vAct)}  - \sum_{i \in \bar{\calE}} \delta{}w^{(j)}_i \bar s_i +
        \sum_{i \in \bar{\calE}} \max \big( 
        \delta{}w^{(j)}_i \beta, \delta{}w^{(j)}_i \alpha
        \big)
        \\ 
        &= \delta\mu_j(\bar \vAct) - \frac{\alpha + \beta}{2} \sum_{i \in \bar{\calE}} \delta{}w^{(j)}_i +
        \sum_{i \in \bar{\calE}} \max \big( 
        \delta{}w^{(j)}_i \beta, \delta{}w^{(j)}_i \alpha
        \big)
    \end{align}
    Now we make use of $\max (a, b) = (a + b + |a-b|)/ 2$ to rewrite
    \[
        \max \big( 
        \delta{}w^{(j)}_i \beta, \delta{}w^{(j)}_i \alpha
        \big) = \frac{(\alpha + \beta)}{2}\delta{}w^{(j)}_i +  \frac{(\beta - \alpha)}{2}|\delta{}w^{(j)}_i|
    \]
    and substituting back to the previous expression we get the result:
    \begin{align}
        \max_{\vs^* \in \calAct_\calE(\vs)}\delta\mu_j ( \vAct^*) &=
        \delta\mu_j(\bar \vAct) - \frac{\alpha + \beta}{2} \sum_{i \in \bar{\calE}} \delta{}w^{(j)}_i +
        \sum_{i \in \bar{\calE}} \left[ \frac{(\alpha + \beta)}{2}\delta{}w^{(j)}_i +  \frac{(\beta - \alpha)}{2}|\delta{}w^{(j)}_i| \right] \\
        &= \delta\mu_j(\bar \vAct) + \frac{\beta - \alpha}{2} \lVert \Deltaj_{\bar{\calE}} \lVert_1
    \end{align}

    \textbf{Bounded from below}.  We first check if there are any positive weights. Because $\calAct$ is only bounded from below, by rewriting the logit difference expression as
    \[
        \delta \mu_j ( \vAct^*)  = \logdiffj + \sum_{i \in \bar{\calE}} \max(\delta{}w^{(j)}_i, 0) (s^*_i - s_i) + \sum_{i \in \bar{\calE}} \max(-\delta{}w^{(j)}_i, 0) (s_i - s^*_i)
    \]
    we can see that the quantity $(s_i^* - s_i)$ is also unbounded. 
    Therefore, any positive weight $\delta{}w^{(j)}_i$ would naturally allow for $\vAct^* \in \calAct$ with fixed $\vAct^*_\calE = \vAct_\calE$ that give positive $\ell\mu_j (\vAct^*)$. 
    If $\exists j \neq \hat y, \exists i \in \bar{\calE}$, such that $\delta{}w^{(j)}_i > 0$, we get  
    \[
        \lim_{s^*_i \to \infty } \delta \mu_j (\vAct^*) = \infty \, .
    \]
    When these conditions are not met, \ie for all $j \neq \hat y$ and $i \in \bar{\calE}$, 
    $\delta{}w^{(j)}_i < 0$,
    the $\mu_j$ expression can be derived from \cref{lemma:label-flip-asymmetric-bounded-domain} and simplifies as\footnote{To be precise, here we are taking the convention here that $0 \cdot (+ \infty) = 0$ as resulting from the limit of $\beta \to \infty$.}
    \begin{align}
        \mu_j &= \logdiffj + \sum_{i \in \bar{\calE}} \left[ \underbrace{(\delta{}w^{(j)}_i)^+}_{=0} (\beta - s_i) + 
        (\delta{}w^{(j)}_i)^- (s_i - \alpha)
        \right] \\
        &= \logdiffj - \sum_{i \in \bar{\calE}} 
        \delta{}w^{(j)}_i (s_i - \alpha)
    \end{align}

    \textbf{Unbounded}. follows directly from the previous case.  Since the domain is unbounded, we have that if $\exists j \neq \hat y, \exists i \in \bar{\calE}$, such that $\delta{}w^{(j)}_i > 0$,
    \[
        \lim_{s^*_i \to \infty} \delta\mu_j (\vAct^*) = \infty
    \]
    or, if $\exists j \neq \hat y, \exists i \in \bar{\calE}$, such that $\delta{}w^{(j)}_i < 0$,
    \[
        \lim_{s^*_i \to -\infty} \delta\mu_j (\vAct^*) = \infty \, .
    \]
    This shows that non-zero weight differences in $\bar{\calE}$ lead to a label flip for unbounded domains.  
\end{proof}

\subsection{The Probabilistic Concept Activations Case}
\label{sec:probability-spaces-appendix}

\cref{prop:label-flip-conditions} only considers independent concept activations.  When modeling concept activations as distributions, activations referring to alternative values of the same categorical concept (\eg different colors) must sum to one, so independence no longer holds.  We study this case separately.

Let $\calAct \subsetneq [0,1]^b$ be the probability simplex for $t > 1$ categorical concepts.  For each concept $c \in [t]$, the number of values it can take are set to $n_c$, such that
\[
    \textstyle
    b = \sum_{c=1}^t n_c \, . 
\]
We indicate with $B_1, \ldots, B_t$ the partitions induced by the $t$ concepts over the $b$ total activations.  
The probability simplex is given by
\[
    \calAct = \bigtimes_{c=1}^t \Delta^{n_c - 1}, 
    \quad
    \Delta^{n_c - 1} := \{ \vz \in [0,1]^{n_c}: \sum_{i \in B_c} z_i = 1 \} \, . 
\]
For a given $\calE \in [b]$, we denote with $F_c := B_c \setminus \calE$ the free coordinates per concept $c$ that are not fixed by $\calE$. 
Because activations in $\calE$ are fixed, the per-concept residual probability mass is given by
\[
    \gamma := \sum_{\calE \in F_c} s_i = 1 - \sum_{i \in B_c} s_i \, .
\]
We use this to prove how $\mu_j$ changes:

\begin{proposition}
    \label{prop:label-flip-conditions-proba}
    Let $\calAct = \bigtimes_{c=1}^t \Delta^{n_c-1}$, let $\vAct \in \calAct$, let $\hat y$ be the argmax of $f(\vAct)$ and let $\calE \subset [b]$ be fixed. 
    Define, for $j \neq \hat y$,
    \begin{equation}
        \mu_j := \logdiffj + \sum_{c=1}^{t} \rho_c, \quad
        \rho_c :=
        \left[ \gamma_c \max_{i \in F_c} \delta{}w^{(j)}_i
        - \sum_{i \in F_c} \delta{}w^{(j)}_i \, s_i 
        \right]
        \label{eq:Mj-simplex}
    \end{equation}
    with the convention $\rho_c = 0$ when $F_c = \emptyset$.  A label flip is possible if and only if $\max_{j \neq \hat y} \mu_j > 0$. 
\end{proposition}

\begin{proof}
    We start from the logit difference expression for $f_j - f_{\hat y}$ and consider a generic $\vs^* \in \calAct_\calE(\vs)$
    \begin{align}
        \delta\mu_j(\vAct^*) 
        &= \logdiffj + \sum_{i \in \bar\calE} \delta{}w^{(j)}_i (s_i^* - s_i) \\
        &= \logdiffj + \sum_{c = 1}^t \sum_{i \in F_c} \delta{}w^{(j)}_i (s_i^* - s_i) \, ,
    \end{align}
    where the last step follows by the fact that only activations in $F_c$ are allowed to differ to the original $\vAct$. Because activations are bounded in each $c \in [t]$, 
    the maximum gain per concept is given by placing all residual probability mass $\mu_c$ to the highest weight difference, \ie
    \[
        \max_{\vAct^* \in \calAct}  
        \sum_{i \in F_c} \delta{}w^{(j)}_i (s_i^* - s_i) = 
        \gamma_c \max_{i \in F_c} \delta{}w^{(j)}_i \, .
    \]
    Therefore, 
    \[
        \max_{\vAct^* \in \calAct_\calE(\vs)} \delta\mu_j(\vAct^*) = \logdiffj + \sum_{c=1}^{t} \left[ \gamma_c \max_{i \in F_c} \delta{}w^{(j)}_i
        - \sum_{i \in F_c} \delta{}w^{(j)}_i \, s_i 
        \right] = \mu_j \, .
    \]
    Clearly, if $\mu_j > 0$ a label flip is possible and, viceversa, if $\mu_j < 0$ no label flip can happen.
\end{proof}

\subsection{Discussion}
\label{sec:extra-theory-discussion}

We briefly discuss our theoretical results.
First of all, the \textbf{unbounded} case recovers \cref{prop:no-apx-for-regular-cbms}: it shows that unless $\calE$ covers all unbounded activations associated to a non-zero weight difference between the prediction $\hat y$ and an alternative class $j \neq \hat y$, then it cannot be a weak \AXp.  Hence, any weak \AXp is forced to cover \textit{all} concept activations linking the prediction to any alternative class through a non-zero weight difference.  The more such activations a CBM admits, the more (weak) \AXp size blows up.  This is precisely what happens with dense CBMs with unbounded logit activations in \cref{sec:q1}.  It also backs up our suggestion to implement \GLasso at training time, as doing so is likely to reduce the number of non-zero weight differences, in turn shrinking the size of weak \AXp{s}.

The \textbf{bounded from below} case recovers both CBMs modeling concepts using ReLU activations and Selective CBMs \citep{schrodi2025selective}.  The logic is the same as for the unbounded case above, except that now \AXp{s} have to cover all concept activations associated to a \textit{positive} weight difference, while \textit{negative} weight differences can be left uncovered.  While this can shrink explanation sizes, CBMs do not control for weight polarity during training.  In contrast, our \GLasso strategy encourages differences to be zero, typically reducing the number of (positive and negative) weight differences between classes.  This yields the improvements reported in \cref{sec:extra-cbm-vs-selective-cbm}.

The \textbf{bounded symmetric} case shows that what matters are the ranges of activations $[\alpha, \beta]$ and the range of weight differences $\delta \vw^{(j)}_i$ for $i \in \bar\calE$.  Smaller activation ranges leave less room for manipulating the concept activations, and smaller weight differences leave less room for translating these into a higher score for the alternative class $j$.  This backs up our discussion of cosine similarity activations in \cref{sec:existing-cbms-suck}, which constrain activations to $[-1, 1]^b$ and are relevant to VLM-CBMs, and also partially motivates our suggestion to model concepts probabilistically, which constrains activations to $[0, 1]^b$.  It also suggests that shrinking the range of the learned weights, which \ENet and \GLasso both implicitly do, is beneficial for constraining the adversary in the bounded case.  The \textbf{bounded asymmetric} case conveys essentially the same message.

Finally, the \textbf{probabilistic} case shows that the exact condition for label flipping is stricter than that for the pure bounded (symmetric and asymmetric) case, meaning that the product-of-simplices structure of the activation space prevents it even more from constructing counterexamples.  Since the concept activation space is fixed, the only knob for reducing the number of possible counterexamples -- and therefore shrinking weak \AXp{s} -- is the magnitude of the learned weights, which both \ENet and \GLasso can shrink.  Further improvements may be gained by designing sparsification penalties targeting the quantities reported in our propositions; we will explore these in future work.

\clearpage
\newpage
\section{Experimental Details}
\label{sec:expermental-details}

We ran all experiments on a 64-core machine with 256 GiB RAM and two RTX 5090 GPUs.  These are used only for computing the CLIP embeddings and nothing else.  Excluding this step, our experiments take about one full day to complete in this setup.

\textbf{Data preparation}.  For each dataset, we use the released train/validation/test partitions; for \CUB, which has no validation split, we carve one via stratified sampling from the training split.  Images are RGB and passed through CLIP's own preprocessing transform without extra resizing or normalization.  We discuss the preparation of each dataset in turn.

For \underline{\CelebA} \citep{liu2015deep}, we follow the pipeline of \citet{debole2026if}.  We translate each of the $40$ binary concept annotations from $\{ \pm 1 \}$ to $\{0, 1\}$, and predict {\tt Male} from the others.  Rather than subsampling, as in \citep{debole2026concepts}, we keep all splits as they are.  We stress that \CelebA has no categorical concepts and therefore \GLasso does not apply to it.

For \underline{\Derm} \citep{kawahara2018seven}, we take the seven dermoscopic checklist criteria from the {\tt meta.csv} file shipping with the dataset and one-hot encode them into $28$ activations.  Following \citep{kawahara2018seven}, we also collapse the $20$ raw diagnoses into $5$ classes  (basal cell carcinoma, melanoma, nevus, seborrheic keratosis, and a miscellaneous class) by matching keywords.  This reproduces their Table I.  Depending on the experiment, \citet{kawahara2018seven} perform one additional round of label merging; we skip this step.

For \underline{\WBC} \citep{tsutsui2023wbcatt} we one-hot encode the $11$ morphological attributes into $31$ activations, and use the released $5$-way cell-type label.

For \underline{\CUB} \citep{wah2011caltech}, we follow \citet{koh2020concept} by replacing per-image concept annotations with class-level concepts, using only data from the training split to determine the aggregation.  \citet{koh2020concept} aggregation step, however, does not treat categorical concepts are as such.  Specifically, it thresholds each attribute independently, which often does not yield a one-hot vector useful for training categorical concepts.  Moreover,  \citet{koh2020concept} filter out activations that do not contribute to at least $10$ classes, as these are too sparse/noisy to be predicted with high accuracy (\citep[Appendix~A]{koh2020concept}), and only $112$ activations survive this stage.  This also removes categories from categorical concepts.  Since we are interested in evaluating the interpretability of CBMs and the benefits of modeling concept activations probabilistically, rather than sheer prediction accuracy, we depart from their processing.  First, we aggregate annotations by taking, within each of the $28$ categorical concepts the value with the largest share of positive votes, with ties broken at random. Moreover, we skip the filtering step.

\textbf{LLM-generated vocabularies}  All VLM-CBM vocabularies are taken directly from the LF-CBM, VLM-CBM and LaBo GitHub repositories and used as-is to compute, for each image, concept activations.
Concretely, we compute the cosine similarity between each image embedding and each concept textual description, both normalized, obtaining concept activation that individually range in $[-1, 1]$.
For the LF-CBM and VLM-CBM vocabularies, we filter out those activations that do not meet a fixed $0.25$ threshold (computed on the training split, to avoid contamination), as done by \citet{oikarinen2023label}, yielding bottleneck sizes of $b = 202$ for the LF-CBM vocabulary and $b = 419$ for the VLG-CBM one.
Since the LaBo vocabulary is pre-filtered via submodular optimization \citep{yang2023language}, we keep all activations.

\textbf{Architectures}.  For all models, the concept probes and dense inference layers are {\tt  sklearn} \citep{pedregosa2012sklearn} logistic regressors fit using L-BFGS \citep{liu1989limited} with default hyperparameters.
\ENet- and \GLasso-sparsified inference layers are fit with FISTA \citep{beck2009fast} with step size $1/L$, where $L$ is the (approximate) Lipschitz constant of the loss, ensuring stability and quick convergence.

\clearpage
\newpage
\subsection{Concept Activation Spaces: In- vs. Out-of-distribution Counterexamples}
\label{sec:extra-activation-spaces}

In all experiments, we employ a model-specific activation space $\calAct$ for both testing whether heuristic explanations are (weak) \AXp{s}, and for constructing \AXp{s} with the exact algorithms.  As anticipated in \cref{sec:existing-cbms-suck,sec:solutions}, we set $\calAct = (-\infty, \infty)^b$ for logit-based CBMs, $\calAct = [0, \infty)^b$ for ReLU CBMs and Selective CBMs, $\calAct = [-1, 1]^b$ for VLM-CBMs with cosine similarity activations, and $\calAct \subseteq [0, 1]^b$ for CBMs modeling concepts probabilistically.

\textbf{Natural vs. out-of-distribution counterexamples}.  In all cases, the adversary can construct counterexamples by manipulating the activations not covered by explanations \textit{independently from each another} -- except for CBMs modelling concepts probabilistically, where activations of the same categorical concept are mutually constrained to sum to $1$.
This setup ignores correlations between concept activations, meaning (weak) \AXp{s} are required to be robust even against ``out-of-distribution'' (OOD) counterexamples that the bottleneck does not naturally output.
While this choice does not invalidate our modeling and algorithmic contributions, we wish to elucidate the rationale behind it.

First, even though OOD counterexamples cannot be generated by the backbone, they can be manually constructed by (possibly ill-intentioned) stakeholders intervening on the model.  By ensuring our activation spaces $\calAct$ match the full range of attainable activation vectors, \textit{we explicitly ensure our (weak) \AXp{s} are robust against all natural and human-constructed counterexamples}.
Second, ensuring an explanation does not admit naturally occurring counterexamples -- with guarantees -- requires inferring the image of the bottleneck, or at least to build a tractable outer approximation thereof.  This is highly non-trivial \citep{singh2019abstract, zhang2018efficient}, making it technically challenging to construct explanations robust only to naturally occurring counterexamples.
Regardless, while focusing on natural counterexamples only downgrades the guarantees that \AXp{s} provide, doing so \textit{can} in principle shrink \AXp size further.  Hence, as mentioned in \cref{sec:conclusion}, we plan to investigate how to do so in future work.

\clearpage
\section{Additional Experiments}
\label{sec:extra-experiments}

In this section, we report the following additional results:
\begin{description}
    \item[\cref{sec:extra-exp-algorithmic-comparison}] Presents an in-depth algorithmic comparison.
    \item[\cref{sec:extra-exp-logits-vs-relus-vs-proba}]  Extends Q1 to logits vs. \textit{ReLU activations} vs. probabilities, including \CelebA.
    \item[\cref{sec:extra-exp-sparsity-on-proba}]  Extends Q2 to probability-based CBMs, including \CelebA.
    \item[\cref{sec:extra-exp-sparsity-on-vlm-cbms}] Extends Q2 to \textit{all} three available VLM-CBM vocabularies.
    \item[\cref{sec:extra-cbm-vs-selective-cbm}] Evaluates \AXp sizes for Selective CBMs \citep{schrodi2025selective}.
\end{description}

\subsection{Extra: Algorithmic Comparison}
\label{sec:extra-exp-algorithmic-comparison}

We compare \OurExactAlgorithm against the original exact algorithm from \citep{marques2023logic}, denoted \OriginalExactAlgorithm, and the three heuristic algorithms -- \TopKAct, \TopKWei, and \TopKProd{} -- along several dimensions:  explanation size, computational cost, and percentage of output explanations that are both faithful and subset-minimal (as per \cref{def:axps}).
All tests are carried out on a random sample of $100$ test instances, kept fixed across experiments.

Since CBMs with raw logits exhibit pathologic behavior, we focus on CBMs that model concepts with either probabilities ($\calAct \subseteq [0, 1]^b$) and VLM-CBMs modeling them as cosine similarities ($\calAct = [-1, 1]^b$), giving all algorithms a chance of producing compact explanations.
\CelebA has no categorical concepts that \GLasso can work with, so we only consider dense CBMs for this dataset.

\subsubsection{\OurExactAlgorithm Outputs Smaller \AXp{s} than \OriginalExactAlgorithm}
\label{sec:extra-exp-ea-vs-ea-plus}

\begin{table}[!h]
    \centering
    \scalebox{0.9}{
    \begin{tabular}{lccccccccccc}
        \toprule
                        &       & \multicolumn{4}{c}{CBM}
                                & \multicolumn{4}{c}{CBM + \GLasso}
        \\
        \cmidrule(lr){3-6}
        \cmidrule(lr){7-10}
                        &       & \multicolumn{2}{c}{\OriginalExactAlgorithm}
                                & \multicolumn{2}{c}{\OurExactAlgorithm}
                                & \multicolumn{2}{c}{\OriginalExactAlgorithm}
                                & \multicolumn{2}{c}{\OurExactAlgorithm}
        \\
        \cmidrule(lr){3-4}
        \cmidrule(lr){5-6}
        \cmidrule(lr){7-8}
        \cmidrule(lr){9-10}
        {\sc Dataset}   & $b$   & {\sc Size}        & {\sc Calls}   & {\sc Size}        & {\sc Calls}   & {\sc Size}        & {\sc Calls}   & {\sc Size}        & {\sc Calls}
        \\
        \midrule
        \CelebA         & $39$  & \npm{20.8}{4.6}   & $39$          & \npm{\textbf{12.6}}{3.7}   & $102$         & --                & --            & --                & --    \\
        \Derm           & $28$  & \npm{17.1}{1.8}   & $28$          & \npm{\textbf{8.3}}{2.0}    & $53$          & \npm{3.1}{2.9}    & $28$          & \npm{\textbf{2.1}}{0.8}    & $33$  \\
        \WBC            & $31$  & \npm{9.9}{3.6}    & $31$          & \npm{\textbf{8.2}}{3.0}    & $60$          & \npm{2.6}{0.9}    & $31$          & \npm{2.4}{0.8}    & $37$  \\
        \CUB            & $312$ & \npm{169.5}{28.5} & $312$         & \npm{\textbf{46.3}}{18.6}  & $221$         & \npm{47.0}{9.9}   & $312$         & \npm{\textbf{23.5}}{3.6}   & $134$
        \\
        \bottomrule
    \end{tabular}
    }
    \caption{\textbf{Comparison between \OriginalExactAlgorithm and our \OurExactAlgorithm} on regular CBMs with prob. activations.  Bold indicates a size improvement of at least one activation.}
    \label{tab:extra-ea-vs-ea-plus-cbms}
\end{table}

\begin{table}[!h]
    \centering
    \scalebox{0.9}{
    \begin{tabular}{lccccccccccc}
        \toprule
                        &       & \multicolumn{4}{c}{VLM-CBM}
                                & \multicolumn{4}{c}{VLM-CBM + \GLasso}
        \\
        \cmidrule(lr){3-6}
        \cmidrule(lr){7-10}
                        &       & \multicolumn{2}{c}{\OriginalExactAlgorithm}
                                & \multicolumn{2}{c}{\OurExactAlgorithm}
                                & \multicolumn{2}{c}{\OriginalExactAlgorithm}
                                & \multicolumn{2}{c}{\OurExactAlgorithm}
        \\
        \cmidrule(lr){3-4}
        \cmidrule(lr){5-6}
        \cmidrule(lr){7-8}
        \cmidrule(lr){9-10}
        {\sc Dataset}   & $b$   & {\sc Size}        & {\sc Calls}   & {\sc Size}        & {\sc Calls}   & {\sc Size}        & {\sc Calls}   & {\sc Size}        & {\sc Calls}
        \\
        \midrule
        \LFCBM          & $202$ & \npm{201.0}{0.8}  & $202$         & \npm{200.9}{0.9}  & $633$         & \npm{36.9}{5.8}   & $202$         & \npm{\textbf{35.3}}{5.9}   & $308$ \\
        \VLGCBM         & $419$ & \npm{416.5}{1.5}  & $419$         & \npm{416.3}{1.5}  & $1286$        & \npm{78.0}{8.4}   & $419$         & \npm{\textbf{66.5}}{11.3}  & $593$
        \\
        \bottomrule
    \end{tabular}
    }
    \caption{\textbf{Comparison between \OriginalExactAlgorithm and our \OurExactAlgorithm} on VLM-CBMs.  Bold indicates a size improvement of at least one activation.}
    \label{tab:extra-ea-vs-ea-plus-vlmcbms}
\end{table}

As a first experiment, we compare \OriginalExactAlgorithm and \OurExactAlgorithm in terms of explanation size and number of oracle invocations, averaged across the $100$ test instances.  Both algorithms are exact in that they are guaranteed to always output (subset-minimal) \AXp{s}, but neither ensures these are size-minimal.
The results on dense and \GLasso-sparsified CBMs are reported in \cref{tab:extra-ea-vs-ea-plus-cbms} and those for VLM-CBMs in \cref{tab:extra-ea-vs-ea-plus-vlmcbms}.

\textbf{\OurExactAlgorithm always produces smaller \AXp{s}}.  Let us begin from explanation size.
The tables indicate that \OurExactAlgorithm always produces smaller explanations compared to the original algorithm.  This is expected considering it \textit{i}) exploits information about the instance and the model to determine a ranking over activations, and \textit{ii}) picks the smallest \AXp for all three possible heuristic rankings, cf. \cref{sec:extra-algorithms}.
The size reduction is more noticeable for dense CBMs.  E.g., for \CUB, it drops from nearly $170$ activations to about $46$.
The improvement is more limited for \GLasso-sparsified CBMs, especially on \Derm and \WBC.  Owing to the benefits of structured sparsification, \OriginalExactAlgorithm already performs very well.  Still, on \CUB under \GLasso, \OriginalExactAlgorithm explanations cover about $47$ activations, while \OurExactAlgorithm only $24$.

The size delta is minimal for dense VLM-CBMs, regardless of vocabulary, due to cosine similarity activations failing to properly constrain the adversary, while a gap forms for sparse VLM-CBMs, with improvements of about $1.5$ and $11.5$ activations depending on the chosen vocabulary.

\textbf{\OurExactAlgorithm makes less than $3 \times$ oracle calls per run}.  Recall that \OurExactAlgorithm runs each ${\tt TopK}^*$ once, each time making a logarithmic number of oracle calls, whereas \OriginalExactAlgorithm calls the oracle exactly $b$ times with no early exit.  The tables show that, in practice, \OriginalExactAlgorithm tends to make fewer total calls in many cases -- \eg dense \CelebA/dense ($39$ vs. $102$), \Derm/dense ($28$ vs. $53$), \WBC/dense ($31$ vs $60$), \Derm/sparse ($28$ vs. $33$), \WBC ($31$ vs. $37$).  The worst offenders are the dense VLM-CBMs, where \OurExactAlgorithm makes as much as three times the number of calls ($202$ vs. $633$ and $419$ vs. $1286$).  This is because concept activations are independent, meaning the {\sc Shrink} step has to test each of the activations covered by the candidate weak \AXp{s} individually.
However, \OurExactAlgorithm makes \textit{fewer} calls in selected cases, and specifically for dense and sparse CBMs on \CUB ($312$ vs. $221$ and $312$ vs. $134$, respectively).

In short, in the worst case, \OurExactAlgorithm can be more demanding than \OriginalExactAlgorithm, but it typically does not pay the entire $3 \times$ cost.
In fact, while \OurExactAlgorithm invokes ${\tt TopK}^*$ three times -- which helps with finding smaller \AXp{s} -- thanks to binary search it can pay less than the full $3 \times$ price.  As we will show next, despite making more oracle calls, \OurExactAlgorithm remains affordable in terms of wall-clock time.

\subsubsection{\OurExactAlgorithm Is Computationally Affordable}
\label{sec:extra-exp-ea-vs-heuristics-time}

Next, we compare the  practical computational cost \OurExactAlgorithm with that of the {\tt TopK} heuristics.  To this end, we measure wall-clock times and the number of oracle calls, averaged across the $100$ test explanations.  Recall that \OurExactAlgorithm is guaranteed to output an \AXp, while the {\tt TopK} heuristics likely don't unless the budget $k$ is sufficiently large.  For fairness, we compare here \OurExactAlgorithm against the {\tt TopK$^*_{bs}$} algorithm (\cref{alg:extra-topk-star-bs}), which selects $k$ automatically to \textit{ensure} the output is a (weak) \AXp and makes use of \textit{binary search} to cut on the number of oracle calls.  This gives it two substantial advantages over heuristics as used in the literature.

\begin{table}[!h]
    \centering
    \scalebox{0.9}{
    \begin{tabular}{lccccccccc}
        \toprule
                        &       & \multicolumn{4}{c}{\sc CBM}             & \multicolumn{4}{c}{CBM + \GLasso}
        \\
        \cmidrule(lr){3-6}
        \cmidrule(lr){7-10}
                        &       & \multicolumn{2}{c}{$\textsc{TopK}^*$}     & \multicolumn{2}{c}{\OurExactAlgorithm}
                                & \multicolumn{2}{c}{$\textsc{TopK}^*$}     & \multicolumn{2}{c}{\OurExactAlgorithm}
        \\
        \cmidrule(lr){3-4}
        \cmidrule(lr){5-6}
        \cmidrule(lr){7-8}
        \cmidrule(lr){9-10}
        {\sc Dataset}   & $b$   & {\sc Calls}   & {\sc Time}    & {\sc Calls}   & {\sc Time}    & {\sc Calls}   & {\sc Time}    & {\sc Calls}   & {\sc Time}
        \\
        \midrule
        \CelebA         & $39$  & $6$           & $0.001$       & $102$         & $0.007$       & --            & --            & --            & -- \\
        \Derm           & $28$  & $6$           & $0.004$       & $53$          & $0.025$       & $6$           & $0.004$       & $33$          & $0.016$ \\
        \WBC            & $31$  & $6$           & $0.005$       & $60$          & $0.040$       & $6$           & $0.006$       & $37$          & $0.026$ \\
        \CUB            & $312$ & $9$           & $0.704$       & $221$         & $16.607$      & $9$           & $0.704$       & $134$         & $10.667$
        \\
        \bottomrule
    \end{tabular}
    }
    \caption{\textbf{Runtime cost of all algorithms for regular CBMs}.  Time is measured in seconds.}
    \label{tab:extra-runtimes-cbms}
\end{table}

\begin{table}[!h]
    \centering
    \scalebox{0.9}{
    \begin{tabular}{lccccccccc}
    \toprule
                        &       & \multicolumn{4}{c}{\sc VLM-CBM}             & \multicolumn{4}{c}{VLM-CBM + \GLasso}
        \\
        \cmidrule(lr){3-6}
        \cmidrule(lr){7-10}
                        &       & \multicolumn{2}{c}{$\textsc{TopK}^*$}     & \multicolumn{2}{c}{\OurExactAlgorithm}
                                & \multicolumn{2}{c}{$\textsc{TopK}^*$}     & \multicolumn{2}{c}{\OurExactAlgorithm}
        \\
        \cmidrule(lr){3-4}
        \cmidrule(lr){5-6}
        \cmidrule(lr){7-8}
        \cmidrule(lr){9-10}
        {\sc Dataset}   & $b$   & {\sc Calls}   & {\sc Time}    & {\sc Calls}   & {\sc Time}    & {\sc Calls}   & {\sc Time}    & {\sc Calls}   & {\sc Time}
        \\
        \midrule
        \LFCBM          & $202$ & $8$           & $0.005$        & $633$        & $0.188$       & $8$           & $0.006$       & $308$         & $0.147$ \\
        \VLGCBM         & $419$ & $9$           & $0.012$        & $1286$       & $1.095$       & $9$           & $0.020$       & $593$         & $1.215$
        \\
        \bottomrule
    \end{tabular}
    }
    \caption{\textbf{Runtime cost of all algorithms for VLM-CBMs}.  Time is measured in seconds.}
    \label{tab:extra-runtimes-vlmcbms}
\end{table}

The results for CBMs in \cref{tab:extra-runtimes-cbms} and those for VLM-CBMs in \cref{tab:extra-runtimes-vlmcbms} tell the same story:  \OurExactAlgorithm makes far more calls to the oracle than the {\tt TopK$^*_{bs}$} algorithm, yet it remains affordable.  The reason behind the extra calls is the checking phase: {\tt TopK$^*_{bs}$} only invokes the oracle a logarithmic number of times (thanks to binary search) for growing an explanation; on the other hand, like \OriginalExactAlgorithm, \OurExactAlgorithm invokes it an additional linear number of times for ensuring subset minimality.  However, we stress that, in the worst case, computing a (compact) \AXp with \OurExactAlgorithm takes $11$s on \CUB for dense CBMs. While slower, it is very much affordable. The benefit, as we will see momentarily, is that \OurExactAlgorithm always outputs smaller explanations than the heuristics.

\subsubsection{\OurExactAlgorithm Finds The Most Compact Explanations}
\label{sec:extra-exp-ea-vs-heuristics-size}

\begin{table}[!h]
    \centering
    \scalebox{0.85}{
    \begin{tabular}{lcccccccc}
        \toprule
                        & \multicolumn{4}{c}{{\sc CBM}}                             & \multicolumn{4}{c}{CBM + \GLasso} \\
        \cmidrule(lr){2-5}
        \cmidrule(lr){6-9}
        {\sc Dataset}   & {\sc Act}         & {\sc Wei}         & {\sc Prod}        & \OurExactAlgorithm    & {\sc Act}         & {\sc Wei}         & {\sc Prod}        & \OurExactAlgorithm \\
        \midrule
        \CelebA         & \npm{35.0}{3.8}   & \npm{15.7}{4.9}   & \npm{29.6}{5.1}   & \npm{\bf 12.6}{3.7}   & --                & --                & --                & -- \\
        \Derm           & \npm{9.5}{2.8}    & \npm{13.4}{5.0}   & \npm{9.7}{2.5}    & \npm{\bf 8.3}{2.0}    & \npm{7.7}{1.8}    & \npm{2.3}{0.7}    & \npm{2.2}{0.9}    & \npm{\bf 2.1}{0.8} \\
        \WBC            & \npm{10.7}{4.3}   & \npm{17.5}{4.8}   & \npm{11.0}{4.9}   & \npm{\bf 8.2}{3.0}    & \npm{8.7}{3.2}    & \npm{3.5}{0.9}    & \npm{4.1}{1.3}    & \npm{\bf 2.4}{0.8} \\
        \CUB            & \npm{50.3}{19.0}  & \npm{84.7}{36.8}  & \npm{55.4}{21.8}  & \npm{\bf 46.3}{18.6}  & \npm{44.4}{14.6}  & \npm{30.0}{5.2}   & \npm{28.2}{4.3}   & \npm{\bf 23.5}{3.6} \\
        \bottomrule
    \end{tabular}
    }
    \caption{\textbf{Comparison between explanation sizes} across datasets for regular CBMs.}
    \label{tab:extra-explanation-size-cbms}
\end{table}

\begin{table}[!h]
    \centering
    \scalebox{0.85}{
    \begin{tabular}{lcccccccc}
        \toprule
                        & \multicolumn{4}{c}{{\sc CBM}}                             & \multicolumn{4}{c}{CBM + \GLasso} \\
        \cmidrule(lr){2-5}
        \cmidrule(lr){6-9}
        {\sc Dataset}   & {\sc Act}         & {\sc Wei}         & {\sc Prod}        & \OurExactAlgorithm    & {\sc Act}         & {\sc Wei}         & {\sc Prod}        & \OurExactAlgorithm \\
        \midrule
        \LFCBM          & \npm{202.0}{0.2}  & \npm{201.8}{0.5}  & \npm{201.8}{0.5}  & \npm{\bf 200.9}{0.9}  & \npm{201.3}{1.3}  & \npm{39.0}{6.8}   & \npm{39.1}{6.8}   & \npm{\bf35.3}{5.9} \\
        \VLGCBM         & \npm{418.6}{0.6}  & \npm{418.1}{1.1}  & \npm{418.2}{1.1}  & \npm{\bf 416.3}{1.5}  & \npm{418.0}{2.1}  & \npm{71.6}{12.2}  & \npm{71.7}{12.2}  & \npm{\bf 66.5}{11.3} \\ 
        \bottomrule
    \end{tabular}
    }
    \caption{\textbf{Comparison between explanation sizes} for \CUB using VLM-CBM vocabularies.}
    \label{tab:extra-explanation-size-vlm-cbms}
\end{table}

Next, we evaluate the average size of the explanations extracted by the various algorithms.
We report the results for regular and \GLasso-sparsified CBMs in \cref{tab:extra-explanation-size-cbms}, marking the smallest average explanations for each model in bold.
As expected, \OurExactAlgorithm -- thanks to the shrink step (\cref{alg:extra-ea-plus}) -- \textit{always} finds smaller explanations, and this holds for both dense and sparse CBMs.

\textbf{Results for CBMs}.  For dense CBMs (\cref{tab:extra-explanation-size-cbms}, left), compared to the runner-up, \OurExactAlgorithm achieves average size $12.6$ (vs. $15.7$ for the runner up, bottleneck size $b = 39$) in \CelebA; $8.3$ (vs. $9.7$, $b = 28$) for \Derm; $8.2$ (vs. $10.7$, $b = 31$) for \WBC; and $46.3$ (vs. $50.3$, $b = 312$) for \CUB.  The difference is in-between $1.4$ activations on average for \Derm and $9$ for \CUB.
For \GLasso-sparsified CBMs (\cref{tab:extra-explanation-size-cbms}, right), the gap tightens, further supporting our usage of \GLasso to contrast \cref{prop:no-apx-for-regular-cbms}.  The only algorithm showing no benefit from \GLasso is \TopKAct, which does not directly depend on the learned weights.
\OurExactAlgorithm still outperforms all heuristics, but (compared to the runner-up) by less than one activation on average on $\Derm$, about $1$ activation in \WBC, and $5$ activations on \CUB.

\textbf{Results for VLM-CBMs}.  We report additional results on dense and \GLasso-sparsified VLM-CBMs in \cref{tab:extra-explanation-size-vlm-cbms}.  Recall that the \LFCBM vocabulary comprises $202$ activations and the \VLGCBM one $419$.  Here, concept activations are cosine similarities, which explains the large explanations output by all algorithms -- and the small gaps between them -- when using dense inference layers (\cref{tab:extra-explanation-size-vlm-cbms}, left):  this is yet another instance of \cref{prop:no-apx-for-regular-cbms} and it is consistent with the poor results reported for cosine similarity activations in \cref{sec:q2}.
The picture is more favorable for sparse CBMs, with explanation sizes dropping sharply for all algorithms (except \TopKAct, as above).
Again, \OurExactAlgorithm is the winner across the board, with an advantage of $4$ activations on average for the \LFCBM vocabulary and of $5$ activations for the \VLGCBM one.

\textbf{Summary}.  While \OurExactAlgorithm provides a modest advantage, size-wise, over the competitors, it is by far the most reliable alternative, as there is no obvious ranking between heuristics, barring poor interactions between \TopKAct with sparsification.

\subsubsection{\OurExactAlgorithm Always Finds \AXp{s}, Heuristics Very Often Do Not}
\label{sec:extra-ea-vs-heuristics-minimality}

We evaluate whether the algorithms are capable of computing explanations that are both faithful (sufficient to defend the prediction) \textit{and} subset-minimal, as required by \cref{def:axps}.  Since all algorithms, including our {\tt TopK} heuristics, always return a \textit{weak} \AXp, as explained in \cref{sec:extra-algorithms}, we only need to check how frequently these are also subset-minimal.

\begin{table}[!h]
    \begin{minipage}{.5\linewidth}
        \centering
        \scalebox{0.8}{
            \begin{tabular}{lccccc}
            \toprule
                            & \multicolumn{3}{c}{\sc TopK} \\
            \cmidrule(lr){2-4}
            {\sc Dataset}   & {\sc Act} & {\sc Wei} & {\sc Prod}    & \OurExactAlgorithm \\
            \midrule
            \CelebA         & $0\%$     & $10\%$    & $1\%$         & ${\bf 100}\%$ \\
            \Derm           & $37\%$    & $2\%$     & $37\%$        & ${\bf 100}\%$ \\
            \WBC            & $9\%$     & $0\%$     & $12\%$        & ${\bf 100}\%$ \\
            \CUB            & $5\%$     & $1\%$     & $5\%$         & ${\bf 100}\%$ \\
            \midrule
            Average         & $12.8\%$ & $3.5\%$   & $13.8\%$       & ${\bf 100}\%$ \\
            \bottomrule
        \end{tabular}
        }
    \end{minipage}%
    \begin{minipage}{.5\linewidth}
        \centering
        \scalebox{0.8}{
        \begin{tabular}{lccccc}
            \toprule
                            & \multicolumn{3}{c}{\sc TopK} \\
            \cmidrule(lr){2-4}
            {\sc Dataset}   & {\sc Act} & {\sc Wei} & {\sc Prod}    & \OurExactAlgorithm \\
            \midrule
            \CelebA         & --        & --        & --            & -- \\
            \Derm           & $0\%$     & $81\%$    & $85\%$        & ${\bf 100}\%$ \\
            \WBC            & $0\%$     & $18\%$    & $22\%$        & ${\bf 100}\%$ \\
            \CUB            & $0\%$     & $1\%$     & $1\%$         & ${\bf 100}\%$ \\
            \midrule
            Average         & $0\%$     & $33.4\%$  & $36\%$        & ${\bf 100}\%$ \\
            \bottomrule
        \end{tabular}
        }
    \end{minipage} 
    \caption{\textbf{Whereas \OurExactAlgorithm always computes \AXp{s}, heuristics do not}. \textbf{Left}: dense CBMs. \textbf{Right}: \GLasso-sparsified CBMs.}
    \label{tab:extra-percent-sm-axps}
\end{table}

We compute the \textit{percentage of subset-minimal explanations} across $100$ test instances and report the results for regular CBMs in \cref{tab:extra-percent-sm-axps}.
The table indicates that, as expected, \OurExactAlgorithm retrieves subset-minimal \AXp{s} for every instance, model and dataset.
The heuristics, on the other hand, never achieve $100\%$ subset-minimality.  For dense CBMs, they attain at best $37\%$ subset-minimality (on \Derm with \TopKAct and \TopKProd), but the percentage can go as low as $0\%$ (on \CelebA with \TopKAct).  The weight-based heuristics, \TopKWei and \TopKProd, fare better on \GLasso-sparsified CBMs, in that they manage to achieve high percentage of subset-minimality on \Derm (up to $85\%$); yet they still fail catastrophically on \CUB, where the best heuristic attains $1\%$ subset-minimality.
In conclusion, these results demonstrate the advantages of \OurExactAlgorithm over common heuristics.

\clearpage
\subsection{Extended Q1: Logits vs. ReLUs vs. Probabilities}
\label{sec:extra-exp-logits-vs-relus-vs-proba}

We report a comparison of all common concept activation modeling strategies -- logits, ReLU activations, probabilities -- in \cref{fig:extra-results-logits-vs-logits-proba}.

\begin{figure}[!h]
    \centering
    \begin{tabular}{ccc}
        \includegraphics[width=0.315\linewidth]{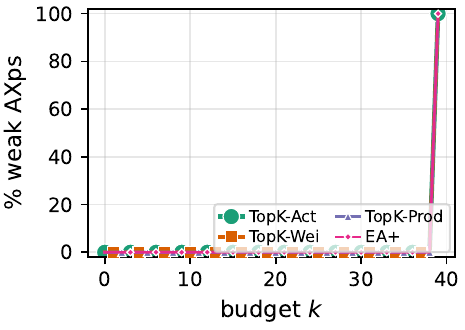}
        \includegraphics[width=0.315\linewidth]{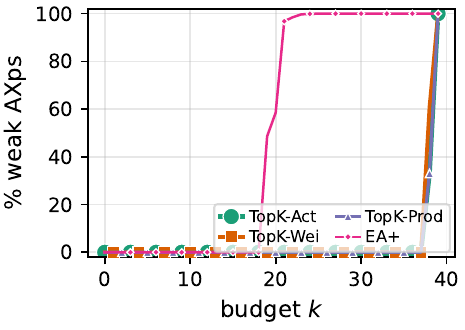}
        \includegraphics[width=0.315\linewidth]{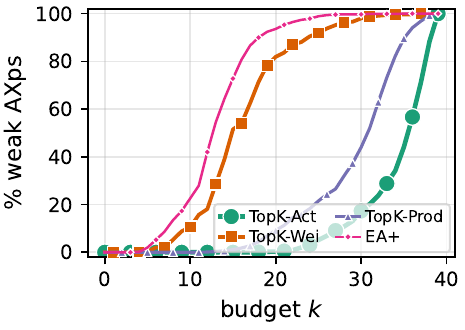}
        \\
        \includegraphics[width=0.315\linewidth]{figures/kstar_cdf_derm7pt_logit_dense.pdf}
        \includegraphics[width=0.315\linewidth]{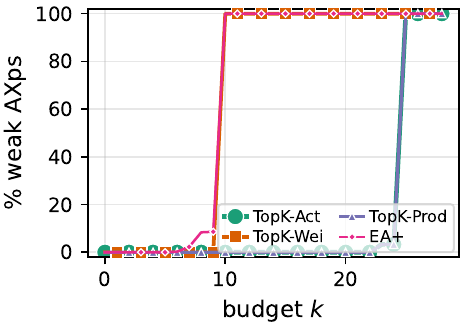}
        \includegraphics[width=0.315\linewidth]{figures/kstar_cdf_derm7pt_prob_dense.pdf}
        \\
        \includegraphics[width=0.315\linewidth]{figures/kstar_cdf_wbcatt_logit_dense.pdf}
        \includegraphics[width=0.315\linewidth]{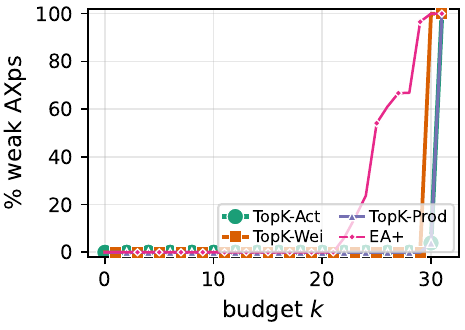}
        \includegraphics[width=0.315\linewidth]{figures/kstar_cdf_wbcatt_prob_dense.pdf}
        \\
        \includegraphics[width=0.315\linewidth]{figures/kstar_cdf_cub_logit_dense.pdf}
        \includegraphics[width=0.315\linewidth]{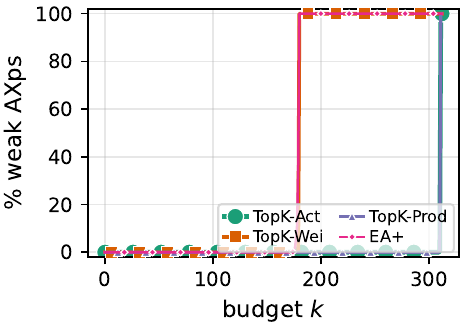}
        \includegraphics[width=0.315\linewidth]{figures/kstar_cdf_cub_prob_dense.pdf}
    \end{tabular}
    \caption{\textbf{Modeling concepts as probabilities shrinks \AXp{s} the most} on all datasets.  From the top: \CelebA, \Derm, \WBC, \CUB). Left to right: logits, ReLU, probabilities.}
    \label{fig:extra-results-logits-vs-logits-proba}
\end{figure}

\textbf{ReLUs sit in-between logits and probabilities}.  All datasets display the same trend:  logits ($\calC = (-\infty, +\infty)^b$) are pathologic, probabilities ($\calC \subseteq [0, 1]^b$) enable more compact \AXp{s} for all algorithms, and ReLU activations ($\calC = [0, \infty)^b$) sit in-between, as shown by the curves in the middle row.  This is perfectly in line with \cref{prop:no-apx-for-regular-cbms}.

\textbf{\TopKAct and \TopKProd underperform on \CelebA}.  This is evident when comparing the curves of the heuristic algorithms in \CelebA (top row) and on the other datasets (other rows).  The reason is simple.  This is because \CelebA only has binary concepts, which means $\calAct = [0, 1]^b$, with no simplex structure, leaving the adversary less constrained than in other datasets.  This directly penalizes heuristics focusing on the highest activations.

\clearpage
\subsection{Extended Q2: \ENet vs. \GLasso for Probabilities}
\label{sec:extra-exp-sparsity-on-proba}

\cref{fig:extra-results-dense-vs-enet-vs-glasso} reports a complete comparison of \ENet vs. \GLasso sparsification for all datasets with categorical concepts, as these are needed to define the groups used by \GLasso.  These include \Derm, \WBC, and \CUB, but not \CelebA.  As in \cref{sec:q2}, we focus on CBMs with probabilistic activations, which are much better behaved than CBMs with logit activations (as per \cref{sec:q1}).

\begin{figure}[!h]
    \centering
    \begin{tabular}{cccc}
         \includegraphics[width=0.315\linewidth]{figures/kstar_cdf_derm7pt_prob_dense.pdf}
         \includegraphics[width=0.315\linewidth]{figures/kstar_cdf_wbcatt_prob_dense.pdf}
         \includegraphics[width=0.315\linewidth]{figures/kstar_cdf_cub_prob_dense.pdf}
         \\
         \includegraphics[width=0.315\linewidth]{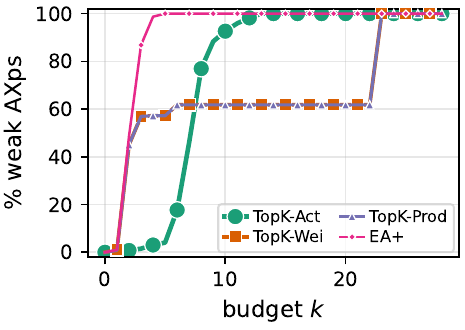}
         \includegraphics[width=0.315\linewidth]{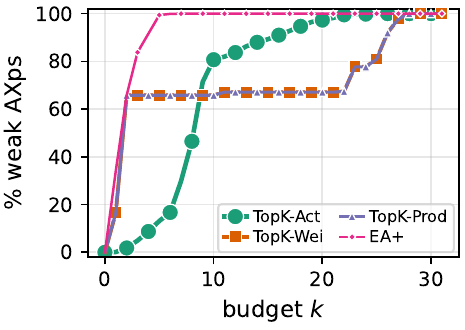}
         \includegraphics[width=0.315\linewidth]{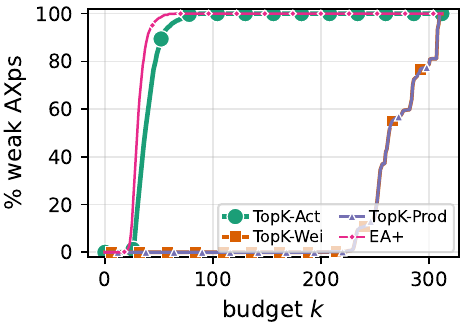}
         \\
         \includegraphics[width=0.315\linewidth]{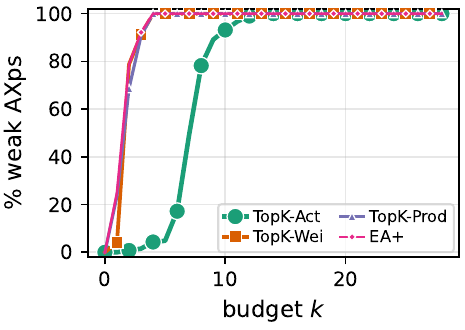}
         \includegraphics[width=0.315\linewidth]{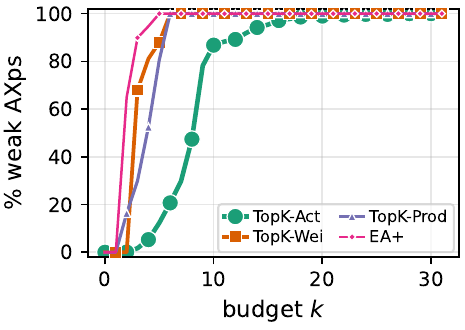}
         \includegraphics[width=0.315\linewidth]{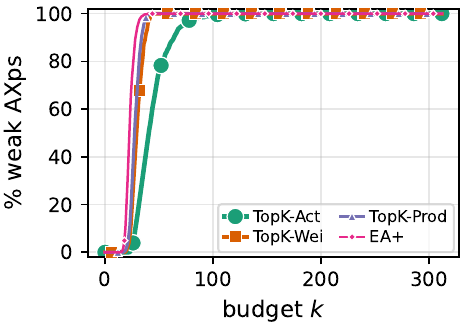}
    \end{tabular}
    \caption{\textbf{\ENet (middle) works partially, \GLasso (below) works consistently.}}
    \label{fig:extra-results-dense-vs-enet-vs-glasso}
\end{figure}

\textbf{\ENet partially works, \GLasso is more consistent.}  The results are compatible with those reported in \cref{sec:q2} for \CUB: while \ENet primarily helps \OurExactAlgorithm, it makes some heuristics -- \TopKWei and \TopKProd -- misbehave as it learns larger non-zero weights.  On the other hand, \GLasso helps more consistently across algorithms.  At the same time, \TopKAct does not benefit from sparsification as much for either penalty.

\clearpage
\subsection{Extended Q2: Sparsity on VLM-CBMs}
\label{sec:extra-exp-sparsity-on-vlm-cbms}

Here, we report the complete FCCs for all VLM-CBM vocabularies.  Recall that concept activations are modelled as cosine similarities.

\begin{figure}[!h]
    \centering
    \begin{tabular}{ccc}
        \includegraphics[width=0.315\linewidth]{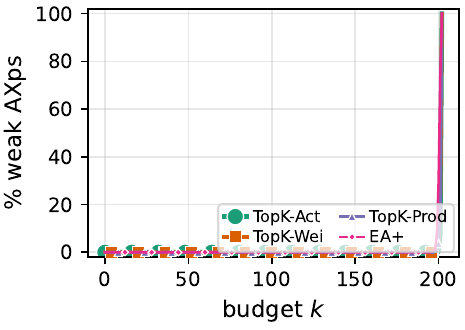}
        \includegraphics[width=0.315\linewidth]{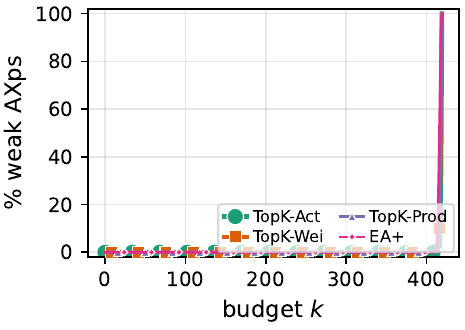}
        \includegraphics[width=0.315\linewidth]{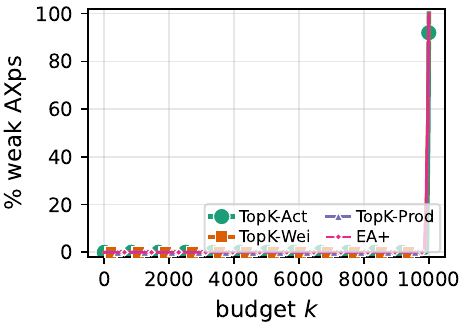}
        \\
        \includegraphics[width=0.315\linewidth]{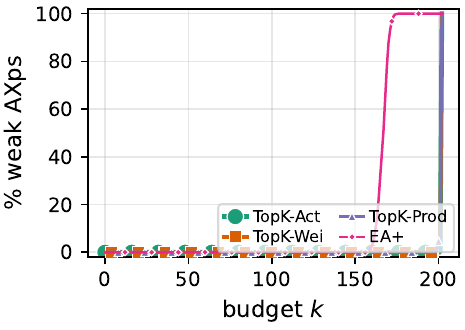}
        \includegraphics[width=0.315\linewidth]{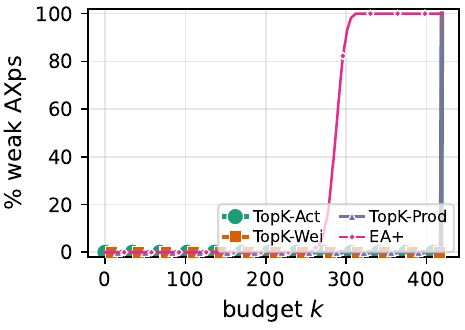}
        \includegraphics[width=0.315\linewidth]{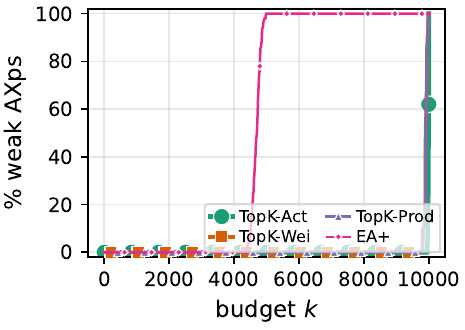}
        \\
        \includegraphics[width=0.315\linewidth]{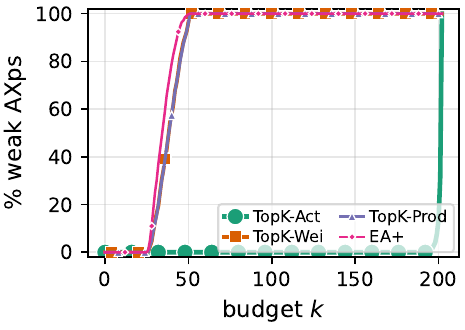}
        \includegraphics[width=0.315\linewidth]{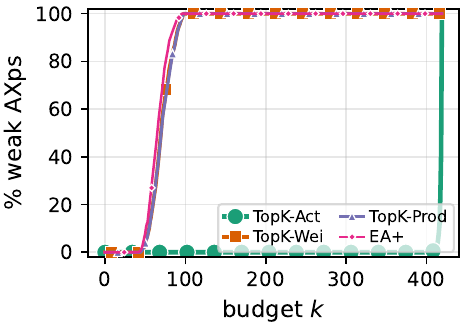}
        \includegraphics[width=0.315\linewidth]{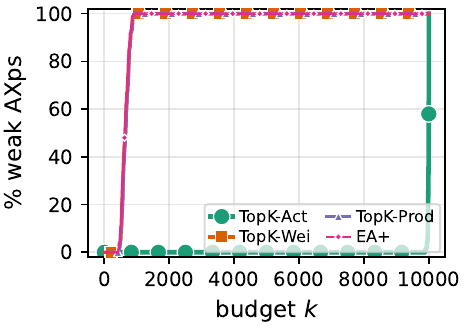}
        \\
    \end{tabular}
    \caption{\textbf{Q3: For VLM-CBMs, \GLasso outperforms other sparsification strategies}, as shown here with \CUB.  Each VLM-CBM employs its own LLM-generated vocabulary, with \LaBo's sporting $10,000$ activations.}
    \label{fig:extra-results-exp-sparsity-on-vlm-cbms}
\end{figure}

\textbf{\GLasso outperforms \ENet on all vocabularies}.  The FCCs in \cref{fig:extra-results-exp-sparsity-on-vlm-cbms} -- dense above, \ENet in the middle, and \GLasso at the bottom -- indicate a large margin in the percentage of explained examples admitting a smaller weak \AXp.  As before, despite employing bounded concept activations, dense VLM-CBMs display pathologic behavior, precisely because their weights are dense.
Introducing \ENet (middle row) helps: focusing on \OurExactAlgorithm for \LFCBM, about $80\%$ explained instances admit (weak) \AXp{s} of size $\sim 160$ (where the full bottleneck holds $202$ activations); for \VLGCBM, of size $\sim 300$ (compared to $419$); for \LaBo, of size $\sim 5{,}000$ (compared to $10{,}000$). The heuristic algorithm, however, do not see any size improvement, backing up our observations in \cref{sec:q2}.
In stark contrast, \GLasso dramatically shrinks explanations of all algorithms except \TopKAct (again, as noted in \cref{sec:q2}).

\textbf{VLM-CBMs resist compactly faithul explanations}.  Despite the substantial improvements given by \GLasso, faithful explanations of size $100$ (for \VLGCBM) and up to $500$ (for \LaBo) are unlikely to be human interpretable.  This supports our key message: large concept bottlenecks should not be (always, at least) equated with reliable interpretability.  We discuss possible ways to overcome this roadblock in \cref{sec:conclusion} and leave an extensive evaluation to future work.

\clearpage
\subsection{Extra: SelectiveCBM}
\label{sec:extra-cbm-vs-selective-cbm}

Compared to standard CBMs, Selective CBMs \citep{schrodi2025selective} sparsify both concept activations (per-input) and weights (across inputs) by combining an instance-dependent gating function with an elastic net penalty.

To assess the benefits of this double sparsification strategy, we compare Selective CBMs against:
\textit{i}) a baseline, namely dense CBMs with ReLU concept activations; and
\textit{ii}) sparsified CBMs (\ENet for \CelebA and \GLasso for \WBC) with logit activations.
We apply them to two datasets, \CelebA (no categorical concepts, \ie{} \GLasso is unnecessary) and \WBC (several categorical concepts, \ie{} \GLasso is useful).
As usual, we maximize $\lambda$ subject to the constraint that $F_1^Y$ is at most $10\%$ worse than the dense baseline, for fairness.

\begin{figure}[!h]
    \centering
    \begin{tabular}{ccc}
        \includegraphics[width=0.315\linewidth]{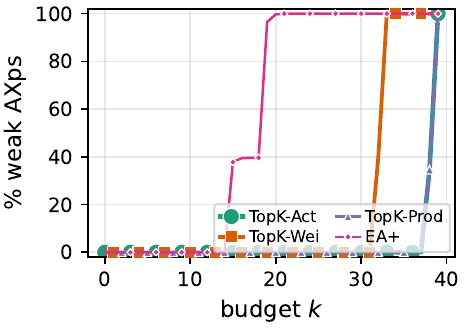}
        \includegraphics[width=0.315\linewidth]{figures/kstar_cdf_celeba_relu_dense.pdf}
        \includegraphics[width=0.315\linewidth]{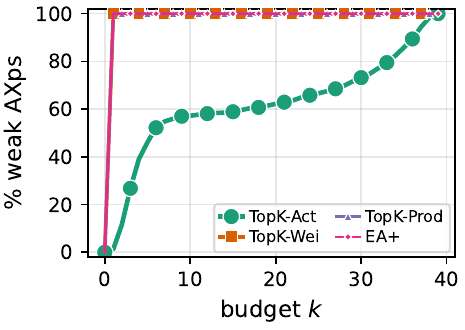}
        \\
        \includegraphics[width=0.315\linewidth]{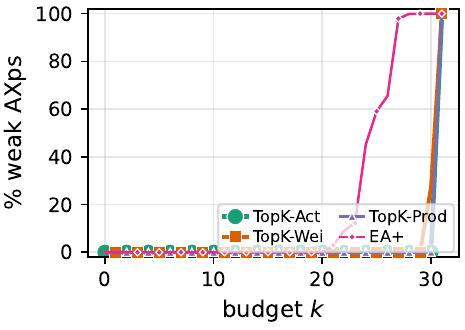}
        \includegraphics[width=0.315\linewidth]{figures/kstar_cdf_wbcatt_relu_dense.pdf}
        \includegraphics[width=0.315\linewidth]{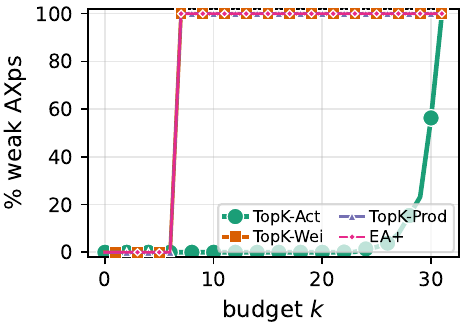}
    \end{tabular}
    \caption{\textbf{Selective CBMs behave like ReLU CBMs and worse than sparsified CBMs} in terms of \AXp size, shown here for \CelebA (top) and \WBC (bottom).}
    \label{fig:extra-results-cbm-vs-selective-cbm}
\end{figure}

\textbf{Selective CBMs behave similarly to ReLU CBMs.}  \cref{fig:extra-results-cbm-vs-selective-cbm} shows that Selective CBMs (left) behave similarly to ReLU CBMs (middle), \ie they provide a moderate improvement to \AXp size compared to regular CBMs (which are pathological, as shown in \cref{sec:q1}).  The improvement applies chiefly to \OurExactAlgorithm (in \textcolor{magenta}{\textbf{pink}}), whereas all heuristics continue to behave poorly.
However, like ReLU CBMs, Selective CBMs underperform compared to appropriately sparsified CBMs even when these employ raw logits (right).  This is most visible for \OurExactAlgorithm, \TopKWei, and \TopKProd, which all display the same trend, whereas \TopKAct lags behind by a fair amount.

These results support our arguments in \cref{sec:existing-cbms-suck,sec:proofs-appendix} that, from the perspective of \AXp{s}, SelectiveCBM's learnable gating function is equivalent to a pure ReLU:  neither prevents the construction of counterexamples whenever the prediction can be flipped by pushing a concept activation $\Act_j$ to $+\infty$.  The only difference is the threshold for constructing the counterexample: for ReLUs it is $0$, for Selective CBMs it is a possibly non-zero positive constant (\ie the learned gating threshold).

\end{document}